\documentclass{article}
 \usepackage{amssymb}
\usepackage{ijcai25}
\usepackage{caption}  
\usepackage{float}    
\usepackage{subcaption}
\usepackage{xcolor}
\usepackage{times}
\usepackage{soul}
\usepackage{url}
\usepackage[hidelinks]{hyperref}
\usepackage[utf8]{inputenc}
\usepackage{graphicx}
\usepackage{amsmath}
\usepackage{amsthm}
\usepackage{booktabs}
\usepackage{algorithm}
\usepackage{algorithmic}
\usepackage{enumitem}
\usepackage{hyperref}
\usepackage[switch]{lineno}
\newtheorem{theorem}{Theorem}

\usepackage{xcolor}

\title{Granular-Ball-Induced Multiple Kernel K-Means}

\author{
    Shuyin Xia$^{1}$,~
    Yifan Wang$^{1}$,~
    Lifeng Shen$^{1}${\thanks{Corresponding Author}},~and 
    Guoyin Wang$^{2}$\\
    \affiliations
    $^{1}$ Chongqing Key Laboratory of Computational Intelligence, 
    Key Laboratory of Cyberspace Big Data Intelligent Security, Ministry of Education, 
    and the Key Laboratory of Big Data Intelligent Computing, College of Computer Science and Technology, Chongqing University of Posts and Telecommunications\\
    $^2$ National Center for Applied Mathematics, Chongqing Normal University\\
    \emails
    xiasy@cqupt.edu.cn,~
    wangyifan421@foxmail.com,~\\ 
    \href{mailto:shenlf@cqupt.edu.cn}{shenlf@cqupt.edu.cn},~
    wanggy@cqnu.edu.cn
}

\begin{document}

\maketitle

\begin{abstract}
Most existing multi-kernel clustering algorithms, such as multi-kernel K-means, often struggle with computational efficiency and robustness when faced with complex data distributions. These challenges stem from their dependence on point-to-point relationships for optimization, which can lead to difficulty in accurately capturing data sets' inherent structure and diversity. Additionally, the intricate interplay between multiple kernels in such algorithms can further exacerbate these issues, effectively impacting their ability to cluster data points in high-dimensional spaces. In this paper, we leverage granular-ball computing to improve the multi-kernel clustering framework.
The core of granular-ball computing is to adaptively fit data distribution by balls from coarse to acceptable levels. 
Each ball can enclose data points based on a density consistency measurement. 
Such ball-based data description thus improves the computational efficiency and the robustness to unknown noises. Specifically, based on granular-ball representations, we introduce the granular-ball kernel (GBK) and its corresponding granular-ball multi-kernel K-means framework (GB-MKKM) for efficient clustering. 
Using granular-ball relationships in multiple kernel spaces, the proposed GB-MKKM framework shows its superiority in efficiency and clustering performance in the empirical evaluation of various clustering tasks. 
\end{abstract}

\section{Introduction}

Clustering, a fundamental task in machine learning, groups data based on inherent similarities \cite{aggarwal2018introduction,duran2013cluster}. Classic methods include K-means \cite{macqueen1967some}, spectral clustering \cite{ng2001spectral}, and density-based clustering \cite{ester1996density}. K-means, widely used for its simplicity, minimizes squared distances to cluster centroids but struggles with nonlinear and arbitrary-shaped data distributions.

To overcome this limitation, various K-means variants have been developed. 
K-means++ \cite{arthur2006k} improves the convergence speed and clustering accuracy by optimizing the selection of initial cluster centers. \cite{mussabayev2023use} proposed an adaptive clustering algorithm that does not require initialization or parameter selection. Although this algorithm somewhat improves robustness to initial values, its performance remains limited when dealing with high-dimensional and non-separable datasets.

Kernel trick is a powerful technique widely used to handle non-linearly separable data. 
For example, kernel K-means (KKM)~\cite{sinaga2020unsupervised} algorithm was proposed by mapping original data to a high-dimensional kernel space, where data can easily be separated linearly, and K-means clustering is subsequently performed. 
In practice, a kernel matrix is required to construct using the inner products of data points in the high-dimensional space, thereby capturing the nonlinear features of the data \cite{abin2015active,yan2023towards,yu2011optimized,girolami2002mercer}. 
However, KKM's performance is highly dependent on the choice of the kernel function, which is often challenging to determine across different datasets. 

To alleviate the limitation of selecting a single kernel function, the multiple kernel clustering (MKC) concept was introduced \cite{chao2021survey,zhao2009multiple,zhang2020deep}. 
The core idea of MKC is to leverage multiple kernel functions to extract complementary information from different feature spaces and optimally fuse these kernel matrices to improve clustering performance. Among MKC algorithms, one of the most popular approaches is the multiple kernel K-means (MKKM) \cite{huang2011multiple}. MKKM learns a weighted combination of multiple kernel functions, and its core optimization objective can be defined as:
\begin{align}
\min_{\gamma \in \Delta} \min_{\mathbf{H} \in \Gamma} \operatorname{Tr} \left[ \mathbf{K}_{\gamma} \left( \mathbf{I} - \mathbf{H} \mathbf{H}^T \right) \right],
\end{align}
where \( \mathbf{H} \) and \( \mathbf{K}_{\gamma} \) represent the clustering partition matrix and the \( \gamma \)-th kernel matrix, respectively. \(\Delta = \{ \boldsymbol{\sigma} \in \mathbb{R}^\omega \mid \sum_{k=1}^\omega \sigma_k = 1, \sigma_k \geq 0, \forall k \}\), where $\omega$ represents the number of kernel functions and $\sigma_k$ is the kernel function's weight. 
MKKM adopts an alternating optimization strategy to update the kernel weights and the clustering partition matrix iteratively. However, this alternating approach easily suffers from the local optima, leading to a degraded overall clustering performance \cite{yao2020multiple,liu2016multiple}.
To further improve the MKKM, \cite{liu2022simplemkkm} proposed SimpleMKKM, which reformulates the min-max formula into a parameter-free minimization problem and designs a streamlined gradient descent algorithm to solve for the optimal solution. Although this method alleviates the local optimum problem, it is still limited as it requires joint optimization over all samples to compute kernel matrices, significantly increasing time and space complexity.

In this paper, we propose an adaptive and robust granular-ball-induced multiple-kernel K-means clustering framework.
It presents the advantages of scalability and clustering performance compared to recent multi-kernel clustering methods. 
Intuitively, granular-ball computing is a data representation technique that can capture complex data structures with a set of balls. 
Using a density-based central consistency measurement \cite{xia2019granular}, granular balls can be adaptively generated from coarse to acceptable fine
granularities. The use of granular balls reduces the amount of data while effectively capturing the distribution information of the data. Based on granular-ball computing, we introduce the granular-ball-induced kernel (GBK) into a multi-kernel K-means clustering framework. Note that the proposed GBK can be play-and-plugged into existing MKKM methods for clustering, significantly improving the efficiency and clustering performance.

\vspace{0.1cm}
In summary, our main contributions are listed as follows:
\begin{itemize}
\item We are the first to introduce a granular-ball-induced kernel (GBK) into multi-kernel K-means clustering that can be easily extended to existing multiple-kernel methods.
\item The proposed GBK effectively reduces storage requirements and computational costs while ensuring superior clustering performance because it allows for reducing the kernel matrix size from ``the number of samples \( n \)'' to ``the number of granular-balls \( m \)'' ($m \ll n$). 
\item Experimental results verify its effectiveness across various commonly used multi-kernel clustering datasets.
\end{itemize}

\section{Related Work}

\subsection{Multiple Kernel Clustering}

To address the issue of single kernel insufficiency, multiple kernel K-means (MKKM), as an algorithm combining multiple kernel learning and K-means clustering, improves clustering performance by integrating multiple kernel functions. Based on classical MKKM, recent studies have made improvements from different perspectives, including selecting kernels based on kernel correlation \cite{gonen2014localized}, introducing local adaptive kernel fusion \cite{liu2016multiple}, optimizing kernel alignment criteria \cite{cortes2012algorithms}, enhancing robustness \cite{tao2018reliable}, simplifying the optimization process \cite{zhang2021late}, and removing noisy kernels \cite{li2023simple}, etc.
For instance, LMKKM \cite{gonen2014localized} was proposed better to capture sample-specific features through local adaptive kernel fusion. RMKKM \cite{du2015robust} adopted the $l_{2}$-norm to improve resistance to outliers, achieving joint optimization of clustering labels, multi-kernel combinations, and memberships through alternating optimization. 
Besides, ONKC \cite{liu2017optimal} was introduced to enhance the representation capability of the optimal kernel. A multi-view clustering method, LF-MVC \cite{10011211}, based on SimpleMKKM was further studied, which optimized a new objective function using an efficient two-step optimization strategy. Moreover, a multiple kernel clustering method with multi-scale partition selection \cite{10529609} was developed that dynamically removes noisy kernels, and significantly improves performance.

One of the main limitations of the above MKKM methods is that they require constructing kernel matrices for all samples and performing optimization, which results in high time and space costs. In this work, we use the distribution information of samples in multiple kernel spaces to enhance efficiency and clustering accuracy.

\subsection{Granular-ball Computing}

Granular-ball computing (GBC) \cite{xia2019granular} is a helpful tool for improving efficiency by covering data at multiple granularities. 
This aligns with the strategy of ``human cognition prioritizing global features before processing details" proposed in \cite{chen1982topological} and the concept of multi-granularity cognitive computing \cite{wang2017dgcc}. 

\begin{figure}[t]
    \centering
    \begin{subfigure}[b]{0.45\textwidth}
    \centering    \includegraphics[width=\textwidth, trim=0 0 0 0, clip]{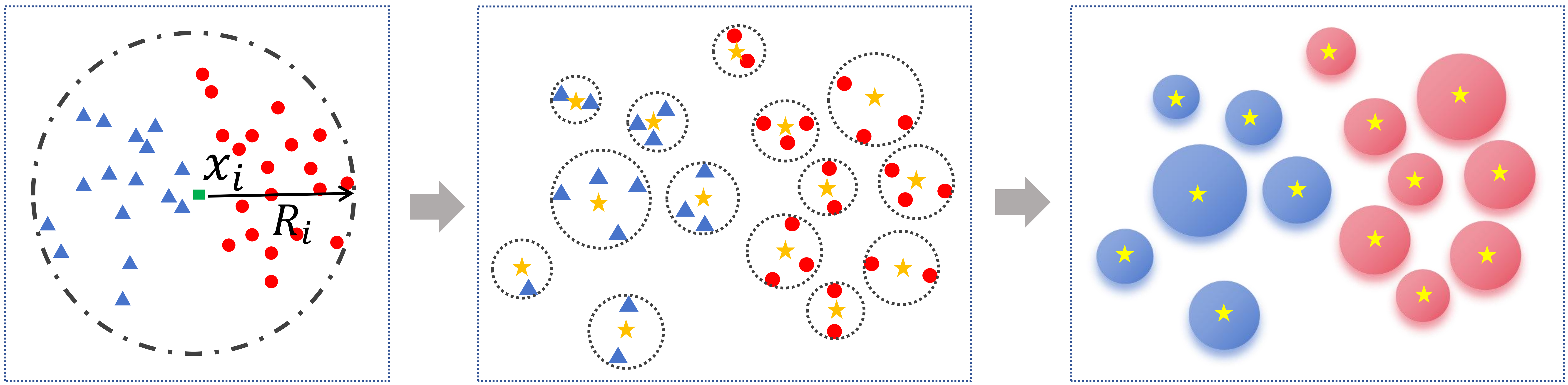}
    \end{subfigure} 
    \caption{\small Illustration of granular ball generation.}
    \label{granular}
\end{figure}

Figure \ref{granular} shows how granular balls cover and represent data sets, which provides an intuitive macroscopic perspective for understanding granular computing.
Granular-ball computing has been studied in various applications, including accelerating unsupervised K-means \cite{xia2020fast}, enhancing density-based clustering like DPC and DBSCAN \cite{cheng2023fast}, \cite{cheng2024gb}, and addressing high-dimensional challenges through granular-ball-weighted K-means and manifold learning \cite{xie2024w,liu2024granular}. 
In spectral clustering, granular-ball methods exploit data structure for efficiency \cite{xie2023efficient}. Recent advancements include the granular-ball clustering algorithm \cite{xia2024gbct}, introducing center consistency for complex data, and fuzzy theory integration to resolve overlapping boundaries from concept drift \cite{xie2024efficient}. Moreover, Granular-ball computing has recently been applied to point cloud registration \cite{hu2025gricp} and multi-view contrastive learning \cite{su2025multi}. 
These successes demonstrate granular balls' powerful representation and generalization capabilities, confirming their potential for multi-view learning. 

\begin{figure*}[t!]
    \centering
    \begin{subfigure}[b]{1\textwidth}
    \centering  \includegraphics[width=0.9\textwidth, trim=0 0 0 0, clip]{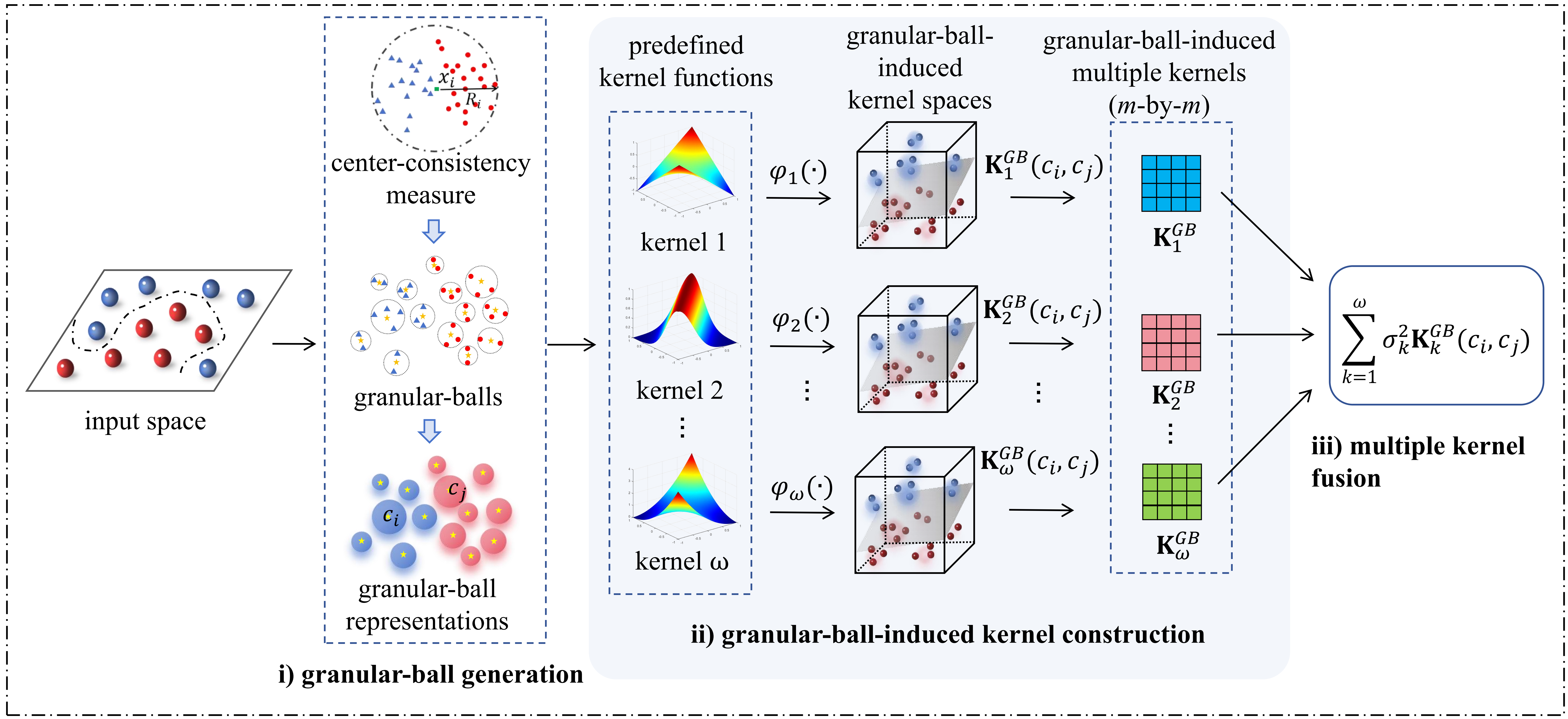}
    \end{subfigure}   
    \caption{{Procedure of granular-ball-induced multiple kernel construction.}}
    \label{process}
\end{figure*}

\section{Methodology}

This section formally elaborates on the proposed granular-ball-induced multiple kernel clustering framework. 
It mainly includes four steps: 
i) the granular-ball generation step, ii) the granular-ball-induced kernel construction, iii) the multiple kernel fusion step, and iv) the optimization of multi-kernel clustering. 
The granular-ball-induced multiple kernel construction procedure (the first three steps) is shown in Figure \ref{process}.
Through the above steps, the proposed method significantly improves the multi-kernel optimization efficiency (smaller number of granular balls) and achieves better clustering performance (granular balls fit the data distribution well and exclude potential noises). 

\subsection{Granular-ball Generation}

Let $\mathcal{D}\!=\!\{\mathbf{x}_u\in\mathbb{R}^d, u\!=\!1, 2, \dots, n\}$ be a dataset where $n$ and $d$ denote the number of samples and the dimensionality of the data, respectively. The subscript $u$ indicates the sample index. 
The granular-ball set is defined as $GB = \{GB_1, GB_2, \ldots, GB_m\}$, where $m$ granular balls cover and describe the dataset $\mathcal{D}$. We use $i$ to denote the $i$-th ball where $i\in\{1,2,\dots,m\}$. The ball $GB_i$'s size is $S_i$. 

\noindent\textbf{Definition 1. Granular-ball Representation:} For the $i$-th $GB_i$, its center $\mathbf{c}_i$ is given by $\mathbf{c}_i = \frac{1}{S_i} \sum_{u=1}^{S_i} \mathbf{x}_u$.
Its radii (maximum radius $r_{\text{max}}$ and average radius $r_{\text{ave}}$) are
\begin{equation}
    r_{\text{max}} = \max_{\mathbf{x}_u \in GB_i}(\|\mathbf{x}_u - \mathbf{c}_i\|), \ \ r_{\text{ave}} = \frac{1}{s_i} \sum_{u=1}^{s_i} \|\mathbf{x}_u - \mathbf{c}_i\|.
\end{equation}

\noindent\textbf{Definition 2. Center-Consistency Measure:} For the $i$-th granular-ball $GB_i$, its center-consistency measure \cite{xia2024gbct} is
\begin{equation}\label{eq:CCM}
    \text{CCM}_{GB_i} = \frac{|\chi_i|}{r_{\text{ave}}} / \frac{S_i}{r_{\text{max}}}=
    \frac{|\chi_i| \cdot r_{\text{max}}}{r_{\text{ave}} \cdot S_i}, 
\end{equation}  
where~$\chi_i = \{\mathbf{x}_u \mid \|\mathbf{x}_u - \mathbf{c}_i\| \leq r_{\text{ave}}\}$. 
Equation (\ref{eq:CCM}) means the density ratio between the average radius and the maximum radius within a granular ball. Intuitively, a smaller ratio means a lower sparsity level of
sample distribution within a granular ball. The criterion for splitting a granular ball is determined by comparing its CCM value with the median CCM of all balls. Specifically, if $\text{CCM}_{GB_i} < \lambda\cdot\text{CCM}_{\text{median}}$ ($\lambda$ is two empirically), the granular ball will be split; otherwise, the splitting process stops.
The algorithm of granular ball generation is provided in Section 1 of the Appendix.

\subsection{Calculation of Granular-ball Kernel}

Kernel methods embed the data into a new feature space via projections to uncover the nonlinear relationships among the data. For the dataset 
\( \mathcal{D}\), 
we consider a set of \( \omega \) projection functions 
\( \Phi = \{\phi_1, \phi_2, \dots, \phi_\omega\} \). 
Each projection \( \phi_k \) $(k=1,2,\dots,w)$ encodes the data into the feature space as a new vector \( \phi_k(\mathbf{x}) \). 
Let \( \{\mathbf{K}_1, \mathbf{K}_2, \dots, \mathbf{K}_\omega\} \) denote the Mercer kernel matrices corresponding to these implicit projections:
\begin{equation}
\mathbf{K}_k(\mathbf{x}_u, \mathbf{x}_v) = \phi_k(\mathbf{x}_u)^T \phi_k(\mathbf{x}_v).
\end{equation}

We embed granular balls into the kernel space and derive the Granular-ball Kernel (GBK), defined as follows: 
\begin{align}\label{eq:bgk_defn}
\mathbf{K}^{GB}_{k}(\mathbf{c}_i, \mathbf{c}_j) =& \phi_k(\mathbf{c}_i)^T \phi_k(\mathbf{c}_j)\nonumber \\
\approx &\left[\frac{1}{S_i} \sum_{u=1}^{S_i} \phi_k(\mathbf{x}_u)\right]^T \left[\frac{1}{S_j} \sum_{v=1}^{S_j} \phi_k(\mathbf{x}_v)\right] \\
=& \frac{1}{S_i S_j} \sum_{u=1}^{S_i} \sum_{v=1}^{S_j} \mathbf{K}_k(\mathbf{x}_u, \mathbf{x}_v),\nonumber
\end{align}%
where $i,j\in \{1,2,...,m\}$. 
Kernel mappings of samples
within the same granular ball are approximately equal due to
their proximity in the input space. This aligns with the local smoothness assumption in kernel methods, where nearby
points have similar kernel mappings. 
Thus, in (\ref{eq:bgk_defn}), the mean of the mappings of the samples $\{\mathbf{x}_u\}$ within the same ball can estimate the granular-ball center $\mathbf{c}_i$'s kernel mapping:  
\begin{align}
\phi_k(\mathbf{c}_i) 
\approx \frac{1}{S_i} \sum_{u=1}^{S_i} \phi_k(\mathbf{x}_u). 
\end{align}

Intuitively, each granular ball can be seen as a small local distribution or cluster, and the kernel mappings of its internal points are expected to concentrate around a common region in the kernel-induced feature space. Let granular balls \( GB_i \) and \( GB_j \) be regarded as distributions \(P\) and \( Q \) respectively, then the granular-ball kernel reflects the expected kernel:
\begin{align}
\mathbb{E}_{\mathbf{x}_u \sim P, \mathbf{x}_v \sim Q} [\mathbf{K}_{k}(\mathbf{x}_u,
\mathbf{x}_v)])\nonumber
\Rightarrow  &
\mathbf{K}_{k}(\mathbb{E}_{[\mathbf{x}_u \sim P]} \mathbf{x}_u, \mathbb{E}_{[\mathbf{x}_v \sim Q]} \mathbf{x}_v)\\
=& \mathbf{K}^{GB}_{k}(\mathbf{c}_i, \mathbf{c}_j). 
\end{align}
Thus, in practice, we have:
\begin{align}
\mathbf{K}^{GB}_{k}(\mathbf{c}_i, \mathbf{c}_j)=\frac{1}{S_{i}S_{j}}  {\sum_{u=1}^{S_{i}}} {\sum_{v=1}^{S_{j}}} \mathbf{K}_{k}(\mathbf{x}_u, \mathbf{x}_v).  
\end{align}

\begin{theorem}
The granular-ball kernel \( \mathbf{K}^{GB}_{k}(\mathbf{c}_i, \mathbf{c}_j)  \) satisfies the Mercer condition. 1. \( \mathbf{K}^{GB}_{k}(\mathbf{c}_i, \mathbf{c}_j)  \) is symmetric. 2. For any 
\( \mathbf{p}\), 
we have \( \mathbf{p}^T \mathbf{K}^{GB}_{k} \mathbf{p} \geq 0 \).
\end{theorem}

\begin{proof}
\textbf{Symmetry}: Since \( \mathbf{K}_k(\mathbf{x}_u, \mathbf{x}_v)= \mathbf{K}_k(\mathbf{x}_v, \mathbf{x}_u)\) satisfies symmetry, it follows that:
\begin{align}
\mathbf{K}^{GB}_{k}(\mathbf{c}_i, \mathbf{c}_j)  
= \frac{1}{S_j S_i} \sum_{v=1}^{S_j} \sum_{u=1}^{S_i} \mathbf{K}_k(\mathbf{x}_v, \mathbf{x}_u)=\mathbf{K}^{GB}_{k}(\mathbf{c}_j, \mathbf{c}_i).\notag
\end{align}
\textbf{Positive Semi-Definiteness}: 
Let $\mathbf{p}=[p_1,p_2,\dots,p_m]$ and $p_i=p'_i/S_i$.
Since $\mathbf{K}_{k}$ is a kernel and it is positive semi-definite, we have \( \mathbf{p}'^T \mathbf{K}_{k} \mathbf{p}' \geq 0 \).  
Then, 
\begin{align}\label{eq:positive}
   \mathbf{p}^T \mathbf{K}^{GB}_{k} \mathbf{p} &=\sum_{i =1}^m \sum_{j =1}^m p_i {p}_j\mathbf{K}^{GB}_{k}(\mathbf{c}_i, \mathbf{c}_j)\notag\\
   &=\sum_{i =1}^m \sum_{j =1}^m\sum_{u=1}^{S_i} \sum_{v=1}^{S_j}\frac{p_i}{S_i}\frac{p_j}{S_j}\cdot\mathbf{K}_k(\mathbf{x}_u, \mathbf{x}_v)\notag\\
   &=\sum_{i =1}^n \sum_{j =1}^n p'_i {p}'_j\mathbf{K}_{k}(\mathbf{x}_i, \mathbf{x}_j)\geq 0.
\end{align} 
Thus, \( K^{GB}_{k}\) satisfies the Mercer condition.
\end{proof}

\subsection{Multiple Kernel Fusion}

Based on the granular-ball kernel \( \mathbf{K}^{GB}_{k}(\mathbf{c}_i, \mathbf{c}_j) \) in Equation (\ref{eq:bgk_defn}), we construct a granular-ball multi-kernel, defined as: 
\begin{align}
    \mathbf{K}_{\boldsymbol{\sigma}}^{GB} = \sum_{k=1}^\omega \sigma_k^2 {\mathbf{K}}_k^{GB}, 
\end{align}
where $\sigma_k^2$ is the kernel weight and \( \boldsymbol{\sigma} \in \Delta \). \(\Delta\) represents the constraints for kernel weights and \(\quad \Delta = \{ \boldsymbol{\sigma} \in \mathbb{R}^\omega \mid \sum_{k=1}^\omega \sigma_k = 1, \sigma_k \geq 0, \forall k \}\). 
Here, \(  \mathbf{K}_{\boldsymbol{\sigma}}^{GB} \) will be used in the subsequent establishment of the granular-ball-induced multi-kernel clustering model.

\begin{algorithm}[t!] \footnotesize
    \caption{GB-SMKKM Algorithm}
    \label{alg:GBMKKM}
    \textbf{Input}: Granular-ball kernel matrix $\{\bar{\mathbf{K}}_k^{GB}\}_{k=1}^\omega$, initial ${t}=1$. \\
    \textbf{Output}: Optimized $\sigma$.

    \begin{algorithmic}[1] 
        \STATE Initialize $\sigma^{(1)} = \frac{1}{\omega}$ and $flag=1$.
        \WHILE{$flag$ holds}
            \STATE Compute $\mathbf{K}^{GB}_{\sigma^{(t)}} = \sum_{k=1}^\omega (\sigma_k^{(t)})^2 \bar{\mathbf{K}}_k^{GB}$ and obtain $\mathbf{H}_{gb}^{(t)}$.
            \STATE Compute $\frac{\partial \mathbf{K}_\sigma}{\partial \sigma_k}$ and obtain descent direction $d^{(t)}$ via gradient descent in Equation (\ref{eq:obj4}).
            \STATE Update $\sigma^{(t+1)} \gets \sigma^{(t)} + \alpha d^{(t)}$.
            \IF{$\max|\sigma^{(t+1)} - \sigma^{(t)}| \leq e^{-4}$}
                \STATE Set $flag=0$.
            \ENDIF
            \STATE $t \gets t+1$.
        \ENDWHILE
        \STATE \textbf{return} $\sigma$.
    \end{algorithmic}
\end{algorithm}

\subsection{GB-induced Multiple Kernel K-means}

In multiple kernel clustering, the kernel matrices and the clustering partition matrix are the keys to clustering optimization. 
The granular-ball clustering partition matrix \( \mathbf{H}_{gb} \in \mathbb{R}^{m \times J} \) is a binary matrix recording the assignment information of granular balls to clusters, defined as: 
\begin{equation}\label{eq:assignment}
\resizebox{.861\linewidth}{!}{$
\displaystyle
\mathbf{H}_{gb}(i, s) =
\begin{cases} 
1, & \text{if}~GB_i \text{ belongs to the } s\text{-th cluster,} \\
0, & \text{otherwise.}
\end{cases} $}
\end{equation}

To balance the influence of different kernels and ensure the rationality of sample assignments, we optimize the model by minimizing the following objective for the kernel coefficient $\sigma_k$ ($k=1,2,...,w$):
\begin{equation}\label{eq:obj1}
\min_{\sigma \in \Delta}   \sum_{k=1}^{\omega} \sigma_k \sum_{i=1}^m \sum_{s =1}^J R_{k,i,s}.
\end{equation}
where $R_{k,i,s}=
  \mathbf{K}_k^{GB}(\mathbf{c}_i, \mathbf{c}_i) - 2\mathbf{K}_k^{GB}(\mathbf{c}_i, \mathbf{z}_s) + \mathbf{K}_k^{GB}(\mathbf{z}_s, \mathbf{z}_s)$.

Here, $J$ clusters are considered. The center of the $s$-th cluster is denoted as $\mathbf{z}_s$. 
In the objective, the first term is the self-similarity of the granular-ball $GB_i$ in the kernel space.
The second term measures the similarity between $GB_i$ and the cluster center $ \mathbf{z}_s$, while the third term denotes the cluster center's self-similarity kernel.

To simplify the objective function, we omit explicit cluster centers in Equation (\ref{eq:obj1}) and reformulate the objective into a matrix form. 
Thereby, we can construct a granular-ball-based multi-kernel clustering model (GB-SMKKM): 
\begin{equation}\label{eq:obj3}
\min_{\sigma \in \Delta} \max_{\mathbf{H}_{gb} \in \Omega} \operatorname{Tr} \left[ \mathbf{K}^{GB}_{\sigma} ( \mathbf{H}_{gb} \mathbf{H}_{gb}^T - \mathbf{I}_m ) \right],   
\end{equation}
where $\Omega = \{ \mathbf{H}_{gb} \in \mathbb{R}^{m \times J} | \mathbf{H}_{gb}^T \mathbf{H}_{gb} = \mathbf{I}_J \}$. 
Further, we have 
\begin{equation}\label{eq:obj4}
\min_{\sigma \in \Delta} \max_{\mathbf{H}_{gb} \in \Omega} \operatorname{Tr} \left[ \mathbf{K}^{GB}_{\sigma} \mathbf{H}_{gb} \mathbf{H}_{gb}^T \right] - \operatorname{Tr} (\mathbf{K}^{GB}_{\sigma} \mathbf{I}_m).
\end{equation}
Algorithm \ref{alg:GBMKKM} presents the GB-SMKKM optimization procedure. 
Section 2 of the {Appendix} provides proof to show that the above objective is differentiable and convex.

\subsection{Computational Complexity}

In this work, $n$ represents the number of samples, $m$ denotes the number of granular balls, $\omega$ denotes the number of kernel functions, and $k$ represents the number of clusters. 
The main computational cost of the proposed method involves the construction of the ($m \times m$) kernel matrix $\mathbf{K}^{GB}_\sigma$ and the optimization of the assignment matrix $\mathbf{H}_{gb}$.
Specifically, for $\omega$ kernel matrices, traversing them for calculation incurs a $\mathcal{O}(\omega m^2)$ complexity. 
Then, for the maximum optimization step and clustering assignment matrix computation, each update of $\mathbf{H}_{gb} \in \Omega$ requires solving a positive semi-definite optimization problem, where the largest eigenvalue is calculated once using K-means clustering. Hence, this step incurs a complexity of $\mathcal{O}(m^3 + m^2 kt)$, where $t$ is the number of optimization iterations. The computational bottleneck lies in the matrix decomposition step, which incurs a $\mathcal{O}(m^3)$ complexity. In contrast, MKKM operates on the complete $n \times n$ kernel matrix, where the matrix decomposition has a $\mathcal{O}(n^3)$ complexity. This highlights the advantage of the granular-ball kernel approach, as it reduces the $\mathcal{O}(n^3)$ term to $\mathcal{O}(m^3)$, where $m \ll n$.  
Finally, when the optimization process for the GB-SMKKM objective function requires $T$ iterations, the total time complexity is: $\mathcal{O}\big(T[\omega m^2 + m^3 + m^2 k t]\big).$

\section{Experiments}
In this section, we conduct the following experiments, including 
evaluating the clustering performance of the proposed granular-ball-induced multiple kernel clustering against recent strong multiple kernel baselines on nine single-view datasets and three multiple-view datasets; comparing different anchor-sampling methods in multi-kernel clustering frameworks to verify the effectiveness of using granular-ball computing;
and analyzing optimization convergence and computational cost. 
The experiments are conducted using MATLAB R2022a \footnote{Source code and supplementary materials are available at: \url{https://github.com/WangYifan4211115/GB-MKKM}}.  

\begin{table}[t!] 
\centering
\resizebox{\columnwidth}{!}{%
\begin{tabular}{@{}lccccc@{}}
\toprule
\textbf{Datasets} & \textbf{\# Samples} & \textbf{Dimension} & \textbf{Cluster} & \textbf{View} & \textbf{Type}\\ 
\midrule
GLIOMA  & 50   & 4,434                & 4  & 1 & gene\\
Srbct   & 63   & 2,308                & 4  & 1 & gene\\ 
YALE    & 165  & 1,024                & 15 & 1 & image\\
JAFFE   & 213  & 676                 & 10 & 1 & image\\ 
ORL     & 400  & 1,024                & 40 & 1 & image\\ 
WDBC    & 569  & 30                  & 2  & 1 & bio.\\
WBC     & 683  & 9                   & 2  & 1 & bio.\\ 
TR41    & 878  & 7,454                & 10 & 1 & text\\ 
DS8     & 1,280 & 2                   & 3  & 1 & noise\\ \midrule
Caltec101-7 & 1,474 & 48,~40,~254,~1984,~512,~928 & 7 & 6 & image\\ 
Mfeat   & 2,000 & 216,~76,~64,~6,~240,~47 & 10 & 6 & image\\ 
Handwritten & 10,000 & 784,~256        & 10 & 2 & image\\ 
\bottomrule
\end{tabular}%
}
\caption{Dataset information.}
\label{tab1}
\end{table}

\subsection{Setup}

\noindent\textbf{Datasets.} We evaluated the proposed algorithm on 12 datasets.
The detailed information of datasets is summarized in Table \ref{tab1}.
As can be seen, these datasets differ significantly in terms of sample size, multi-view attributes, dimensionality, and the number of categories, making them highly representative and capable of comprehensively verifying the applicability and robustness of the algorithm. 

\noindent\textbf{Metrics.} We adopt three widely-used external evaluation metrics  \cite{rand1971objective} to measure clustering performance: 
i) Adjusted Rand Index (ARI), which evaluates the similarity between the predicted and true cluster assignments while adjusting for chance; ii) Normalized Mutual Information (NMI), which quantifies the amount of information shared between the predicted clusters and the ground truth; 
and iii) Clustering Accuracy (ACC) directly measures the proportion of correctly classified samples. To ensure the reliability of the results and reduce the impact of randomness, each method is repeated 20 times, and the average results are reported.

\noindent\textbf{Baselines.} 
In the experiments, the following recent strong baselines are included for comparison.
\begin{itemize}[left=0pt]
\item Approximate kernel K-means (RKKM-a) approximates kernel K-means using a low-rank matrix decomposition, significantly reducing the computational cost.(\cite{Chitta2011ApproximateKK}, SIGKDD)  
\item Affinity aggregation for spectral clustering (AASC) aggregates affinity matrices to enhance spectral clustering. (\cite{huang2012affinity}, CVPR)  
\item Multiple kernel K-means (MKKM) combines multiple kernels into a consensus kernel to enhance alignment and performance. (\cite{huang2011multiple}, IEEE TFS)  
\item Robust multiple kernel K-means (RMKKM) leverages the \( L_{2,1} \)-norm to handle noise and outliers. (\cite{du2015robust}, IJCAI)
\item Simple multiple kernel K-means (Simple-MKKM) simplifies optimization by introducing a streamlined min-max formulation with fewer parameters. (\cite{liu2022simplemkkm}, IEEE TPAMI) 
\item Multiple kernel K-means clustering with simultaneous spectral rotation (MKKM-SR) integrates spectral rotation with multiple kernel K-means to align the clustering structure across kernels. (\cite{lu2022multiple}, ICASSP) 
\item Scalable multiple kernel clustering (SMKC) scales multiple kernel clustering by learning consensus clustering structures from expectation kernels. (\cite{liangscalable}[2024], ICML) 
\end{itemize}

Furthermore, we introduce the proposed granular-ball-induced kernel into the baselines (MKKM, MKKM-SR, and SMKKM). 
For a fair comparison, we uniformly employ three types of kernels (with various parameters): Linear, Polynomial, and Gaussian, resulting in 6 kernel functions. It is noteworthy that for the comparison with RMKKM, we adopt the optimal parameters and 12 kernel functions as used in their original paper \cite{du2015robust}.

\subsection{Results on Single-view Clustering}
Table \ref{tab2} reports comparison results averaged on nine single-view clustering tasks regarding ACC, NMI, and ARI metrics. 
As can be seen, the proposed GB-SMKKM consistently achieves the best results in the three metrics.

\begin{table}[h!]
\centering
\resizebox{0.9\columnwidth}{!}{%
\begin{tabular}{lccc}
\midrule[1pt]
\textbf{Method}     & \textbf{ACC$_{avg}$}(\%) & \textbf{NMI$_{avg}$(\%)} & \textbf{ARI$_{avg}$}(\%)\\ \midrule
Random Selection    & 58.84      & 45.12     &28.87           \\ 
K-means     & 64.73      & 52.65     &40.78           \\ 
BKHK     & 67.21     & 58.40      &48.10  \\ 
RKKM-a     & 62.10      & 50.42     &66.61           \\ 
AASC     & 49.48     & 31.19      &52.88  \\ 
RMKKM    & 71.98   & 62.58      &75.36 \\ \midrule
MKKM   & 72.79     & 62.91      &78.30  \\ 
GB-MKKM  & \textbf{73.84}   & \textbf{66.38}  &\textbf{79.63}     \\ 
{\color{blue} \texttt{improvements} $\uparrow$}  & {\color{blue}{1.05}} & \color{blue} 3.47 & \color{blue} 1.33\\\midrule
MKKM-SR   & 66.54   & 51.58  &71.35   \\ 
GB-MKKM-SR  & \textbf{71.93}  & \textbf{63.37}  & \textbf{77.19}  \\ 
{\color{blue} \texttt{improvements} $\uparrow$}  & \color{blue} 5.59 & \color{blue}11.79 & \color{blue}5.84\\ \midrule
SMKKM  & 73.71  & 62.58  &77.97   \\ 
GB-SMKKM  & \textbf{80.54}  & \textbf{76.51}  &\textbf{83.90} \\ 
{\color{blue} \texttt{improvements} $\uparrow$}  & \color{blue}6.83 & \color{blue}13.93& \color{blue}5.93\\ 
\midrule
\end{tabular}
}
\caption{Comparison results in terms of ACC, NMI, and ARI averaged on nine datasets.}
\label{tab2}
\end{table}

\begin{figure*}[t]
    \centering
    \begin{subfigure}[b]{0.47\textwidth} 
    \centering    \includegraphics[width=\textwidth, trim=0 0 0 0, clip]{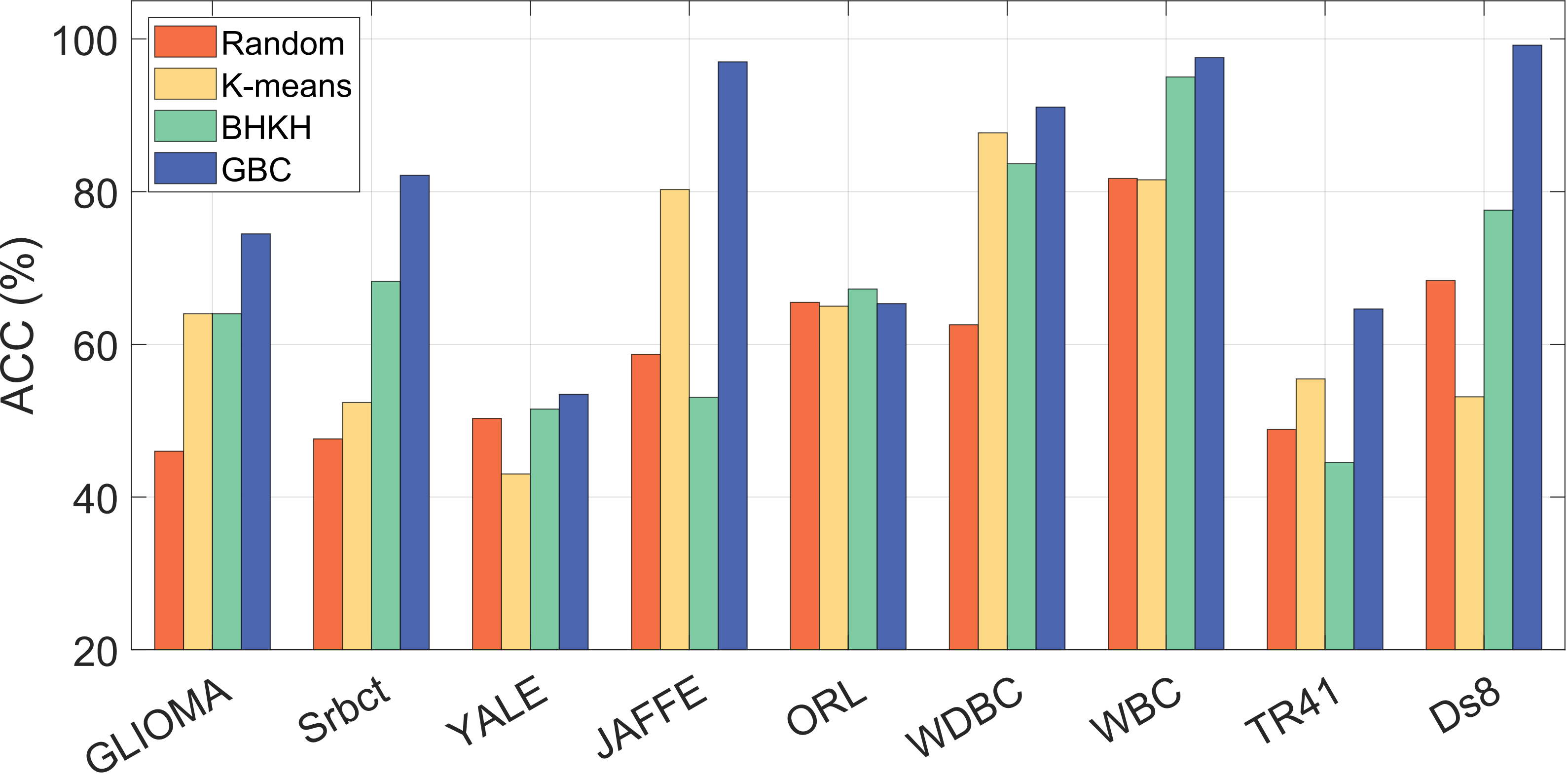}
    \end{subfigure}
    \begin{subfigure}[b]{0.47\textwidth}
    \centering    \includegraphics[width=\textwidth, trim=0 0 0 0, clip]{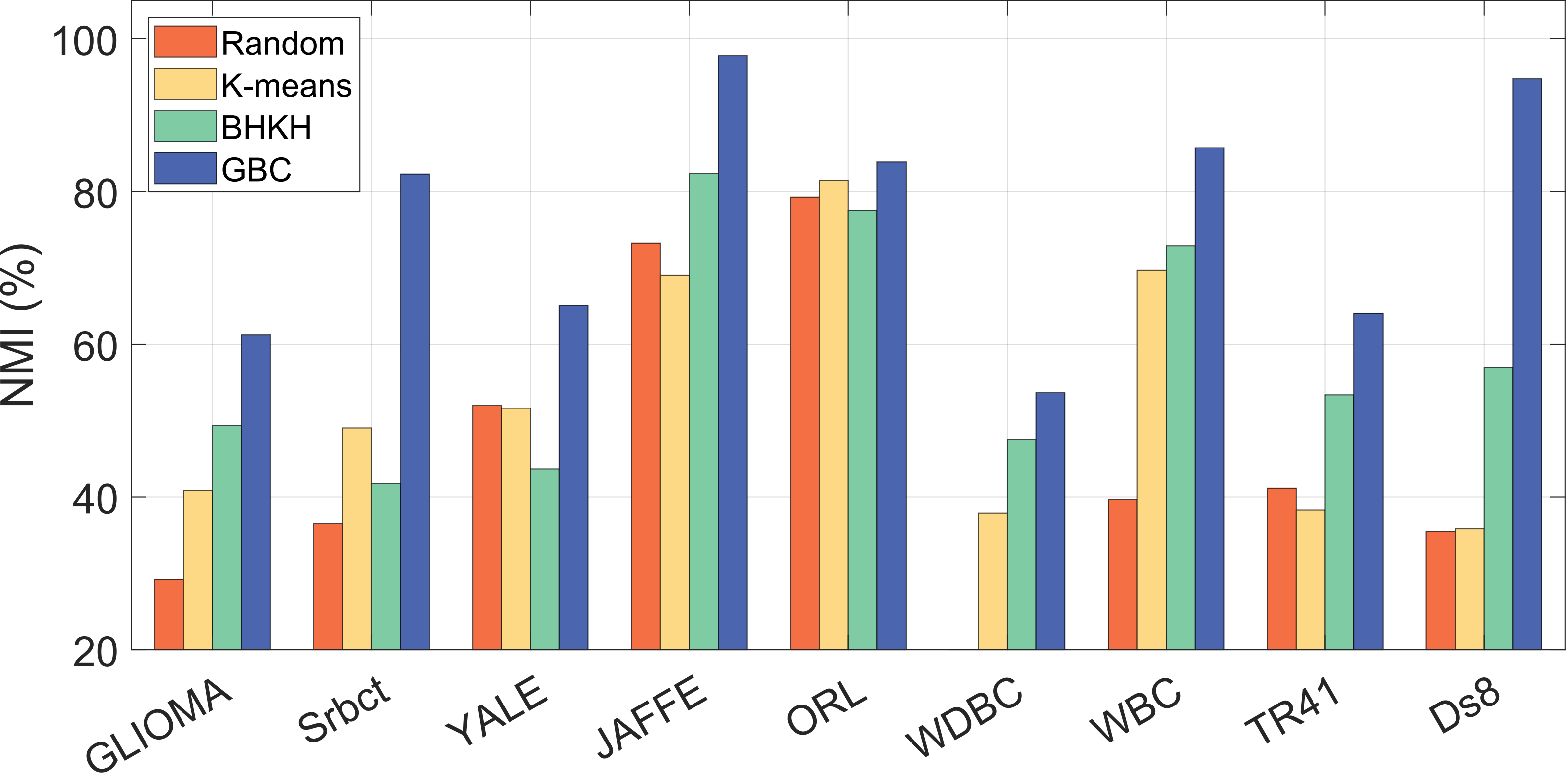}
    \end{subfigure}
    \begin{subfigure}[b]{0.47\textwidth}
    \centering    \includegraphics[width=\textwidth, trim=0 0 0 0, clip]{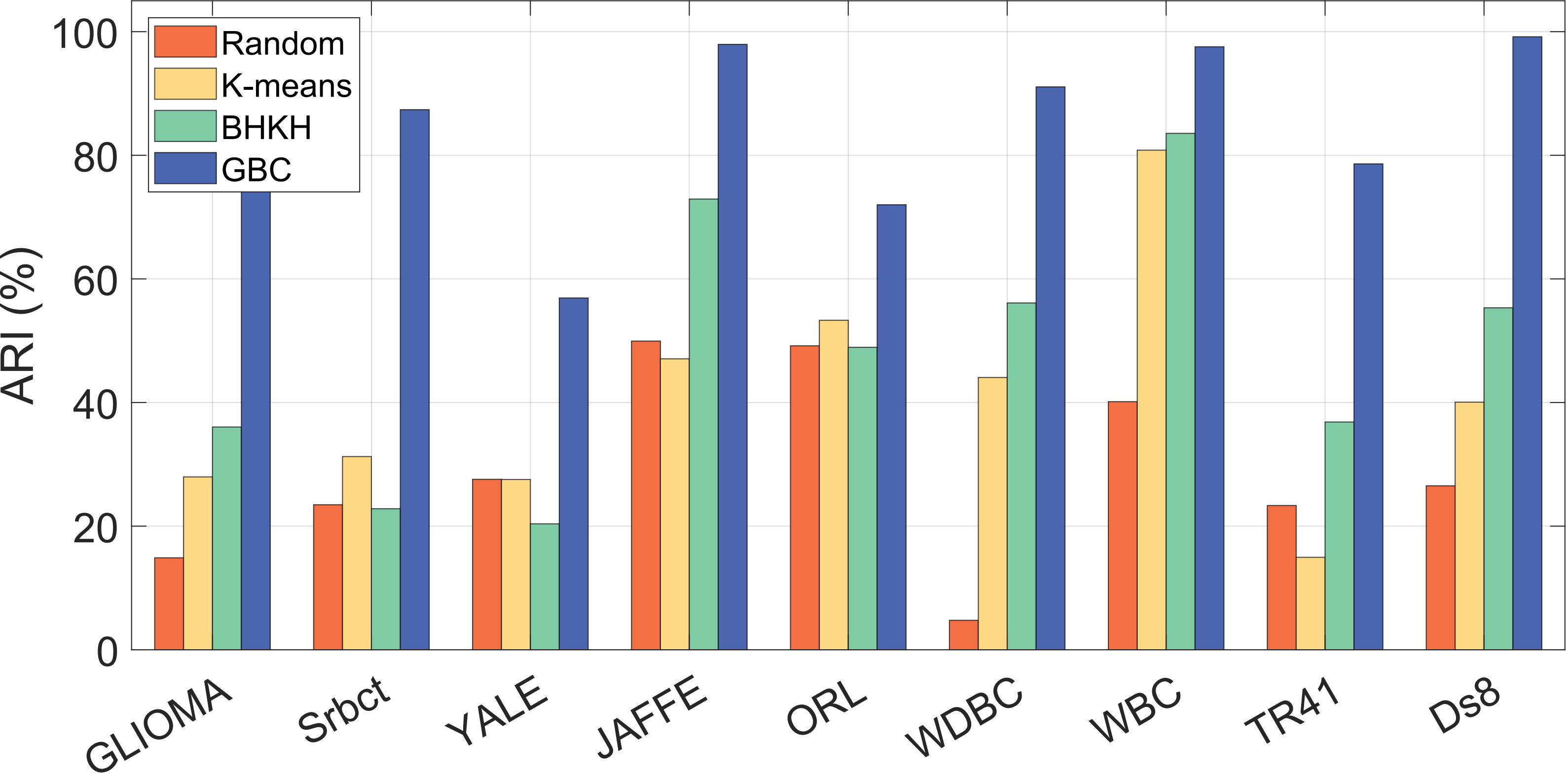}
    \end{subfigure}
    \begin{subfigure}[b]{0.47\textwidth}
    \centering    \includegraphics[width=\textwidth, trim=0 0 0 0, clip]{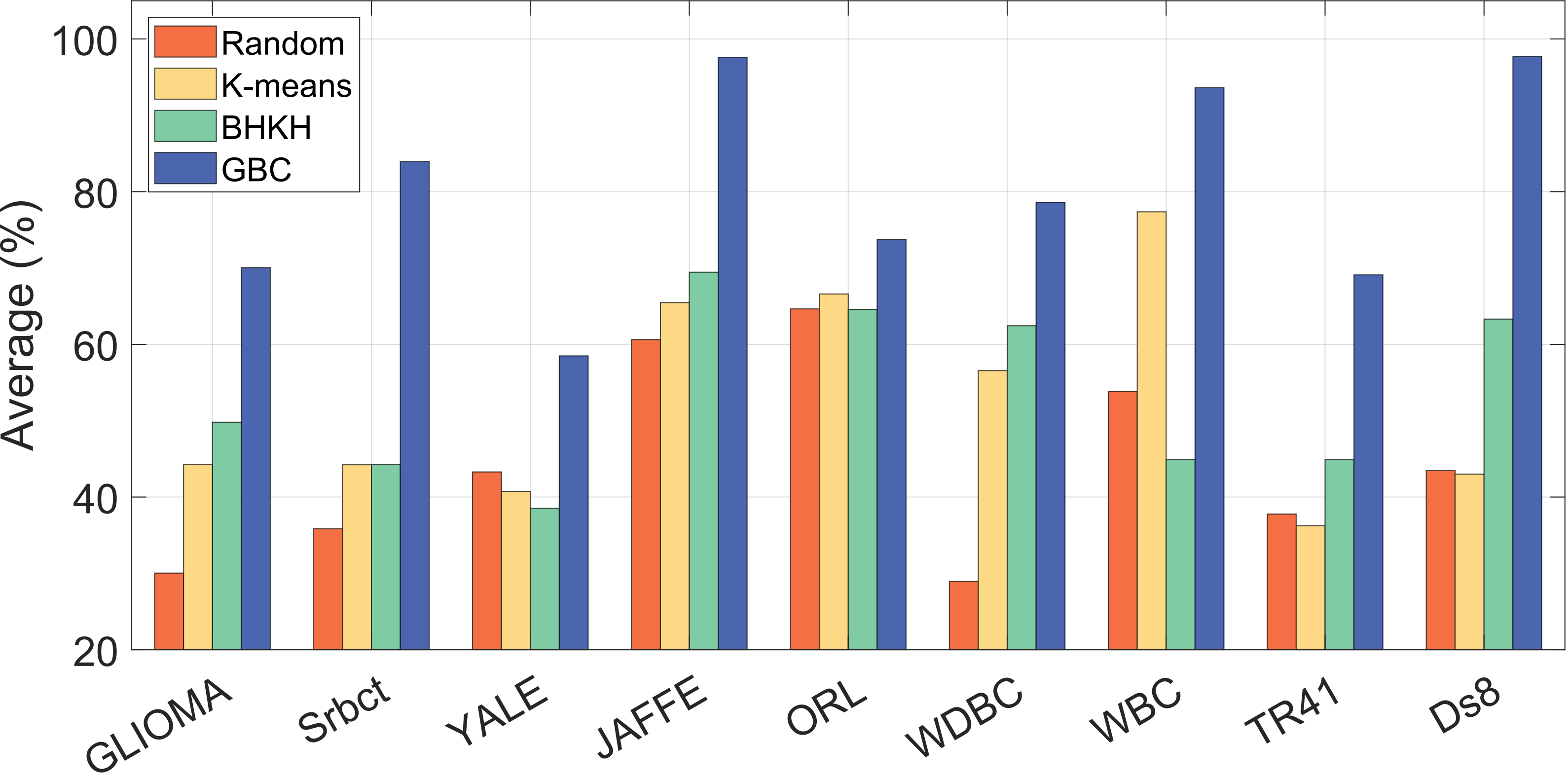}
    \end{subfigure}
    \caption{Comparison of ACC, NMI, ARI, and Average performance between three anchor selection strategies and GBC.} \label{ACC, NMI, ARI and Average}
\end{figure*}

Specifically, compared to SMKKM, the ACC, NMI, and ARI improvements by the proposed method are 6.83\%, 13.93\%, and 5.93\% respectively. 
These improvements are statistically significant as the
paired T-test yields P-values of 0.038, 0.014, and 0.022 for
ACC, NMI, and ARI, respectively. 
Moreover, the improvements are still significant when the proposed granular-ball-induced kernel is introduced into other baselines, such as MKKM and MKKM-SR. 
For example, the average improvements of ACC, NMI, and ARI of GB-MKKM-SR are 5.59\%, 11.79\%, and 5.84\% higher than those of MKKM-SR, respectively, and the upgrades of ACC, NMI, and ARI of GB-MKKM are 1.05\%, 3.47\%, and 1.33\%, respectively. 
These indicate that using granular balls is beneficial in improving the optimization in the multi-kernel spaces. 
This might be explained by the fact that granular balls can capture the data distribution from multiple granularities and exclude potential noisy data by estimated ball boundaries.

We also conducted a statistical analysis regarding ACC using the Friedman Test. 
Friedman test results are
reported in Figure \ref{Friedman Test}. 
Section 3 of the Appendix, the detailed calculation process of the Friedman Test, is added.

\begin{figure}[h!]
    \centering
    \includegraphics[width=0.45\textwidth]{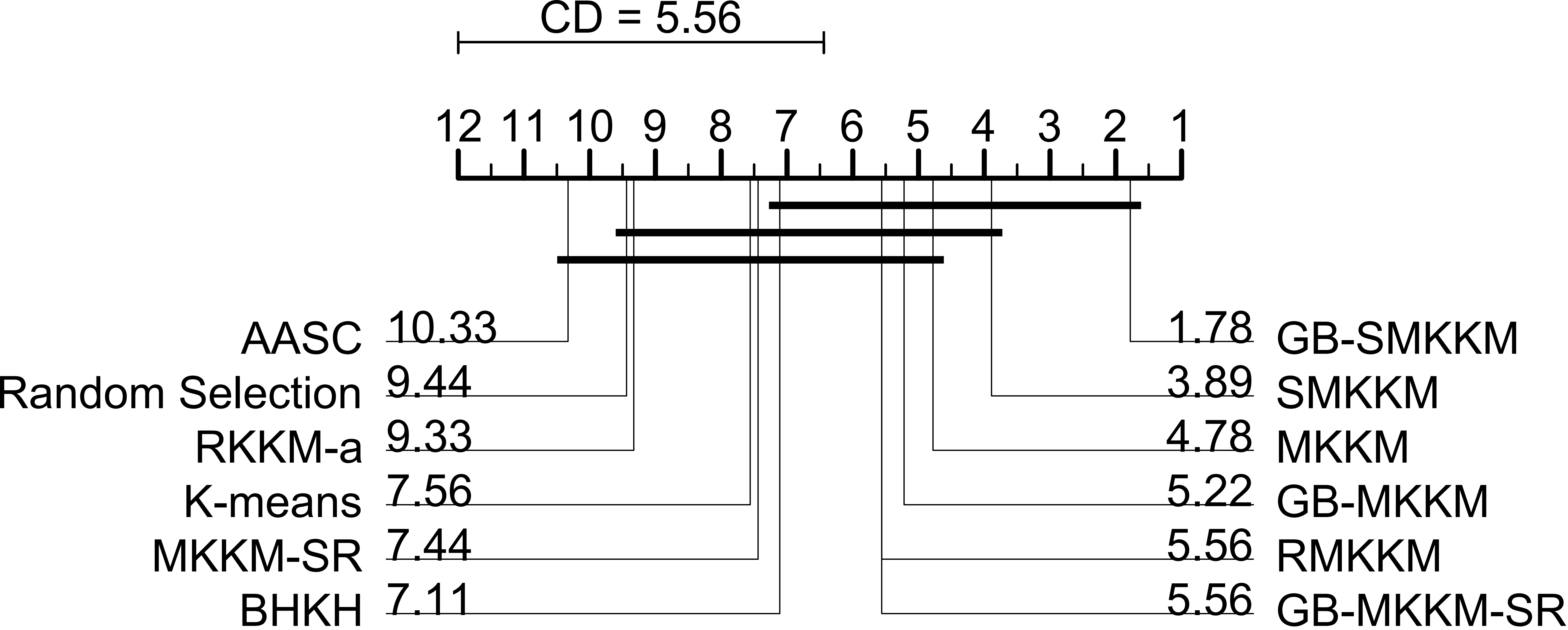}
    \caption{{Comparison on the Friedman test. Groups of methods that are not signiﬁcantly different (at p = 0.05) are connected.}}
    \label{Friedman Test}
\end{figure}

\subsection{Results on Multi-view Clustering}
Here, three multi-view clustering tasks are conducted by comparing a recent strong multi-view baseline SMKC (\cite{liangscalable}[2024], ICML) for comparison. 
Note that SMKC is also built upon the SMKKM, which shares a foundation similar to our model. Results are summarized in Table \ref{SMKC}. 
In the experiments using granular-ball kernels in SMKC, we introduce our variant GB-SMKC. This is achieved by using granular-ball computing to replace their anchor-selection method. 
Also, the number of anchors in SMKC is kept close to the number of granular balls to ensure fairness. 

\begin{table}[t!]
\begin{flushleft} 
\scriptsize
\textbf{Note: }In the table, Caltec101-7 is abbreviated as C101-7.
\end{flushleft}
\centering

\resizebox{\columnwidth}{!}{%
\begin{tabular}{c|ccc|ccc}
\midrule 
\textbf{Method} & \multicolumn{3}{c}{\textbf{GB-SMKC}} & \multicolumn{3}{|c}{\textbf{SMKC}} \\ \midrule 
\textbf{Dataset} & \textbf{Hdigit} & \textbf{Mfeat} & \textbf{C101-7} & \textbf{Hdigit} &\textbf{Mfeat} & \textbf{C101-7}\\ \midrule 
GB/Anchor Counts & \textbf{396} & \textbf{60} & \textbf{30} & 400 & 62 & 31\\
ACC(\%) & \textbf{77.28} & \textbf{94.80} & \textbf{47.83} & 76.63 & 89.40 & 43.69\\ 
NMI(\%) & \textbf{65.48} & \textbf{88.90} & \textbf{38.33} & 64.40 & 80.92 & 36.09\\ 
ARI(\%) & \textbf{77.28} & \textbf{94.80} & \textbf{82.02} & 76.63 & 89.40 & 81.41\\ 
\midrule 
\end{tabular}%
}
\caption{Comparison of GB-SMKC and SMKC.}
\label{SMKC}
\end{table}

As seen in Table \ref{SMKC}, GB-SMKKM still outperforms SMKC in terms of all ACC, NMI, and ARI metrics. Specifically, the ACC of GB-SMKC is higher by 0.56\%, 5.4\%, and 4.14\% on the Handwritten, Mfeat, and Caltech101-7 tasks, respectively. 
This validates the effectiveness and applicability of the GB-induced framework on these multi-view clustering tasks.

\begin{table}[h!] 
\centering
\footnotesize
\resizebox{.46\textwidth}{!}{
\begin{tabular}{lccccc}
\midrule
\textbf{Method} & \textbf{ORL} & \textbf{WDBC} & \textbf{WBC} & \textbf{TR41} & \textbf{Ds8} \\\midrule
RKKM-a        & 0.164         & 0.104       & 0.073         & 0.144         & 0.365 \\
AASC          & 0.942        & 0.216       & 1.772         & 2.962         & 1.524 \\
RMKKM         & 2.338       & 1.605       & 1.031         & 3.769         & 7.713 \\\midrule
MKKM         & 0.067        & 0.104       & 0.116         & 0.152         & 0.644 \\
GB-MKKM     & \textbf{0.021}   & \textbf{0.017} & \textbf{0.008} & \textbf{0.033} & \textbf{0.154} \\\midrule
MKKM-SR       & 0.322       & 0.056       & 0.048         & 0.182         & 0.711 \\
GB-MKKM-SR  & \textbf{0.196}  & \textbf{0.014} & \textbf{0.015} & \textbf{0.047} & \textbf{0.041} \\\midrule
SMKKM         & 0.111          & 0.542       & 2.175         & 0.202         & 6.504 \\
GB-SMKKM   & \textbf{0.050}   & \textbf{0.536} & \textbf{0.220} & \textbf{0.073} & \textbf{1.202} \\
\midrule
\end{tabular}
}
\caption{Running time (s) comparison of the GB-induced framework and six baselines.}
\label{time}
\end{table}

\subsection{Effects of Anchor Selection Strategies}

Granular-ball computing reduces the computational cost by transforming the calculation of point-point relationships into that of ball-ball relationships. 
As it adaptively determines subsets, Granular-ball computing can be regarded as an anchor selection method. 
Compared to other anchor-sampling methods (e.g., Balanced K-means based Hierarchical K-means(BKHK) \cite{zhu2017fast}, K-means \cite{cai2014large}, and random anchor selection \cite{nie2023fast}), granular-ball computing can be better at describing data distribution with complex shapes.  
This is because the generated granular balls adaptively determine multi-granularity subsets based on data distribution characteristics and metric criteria. The ball boundaries can exclude unknown noise points, further improving robustness. 
In this experiment, we compare the above anchor-sampling methods to provide empirical evidence to verify the effectiveness. 

Figure \ref{ACC, NMI, ARI and Average} summarizes the comparison results. 
As can be seen, granular-ball computing is a better choice than all other sampling methods for multi-kernel clustering optimization. 
This is because the procedure of granular-ball generation is based on data distribution. 
The generated balls better fit the original data distribution through the density-based consistency metric. 
Moreover, the ball boundaries can exclude unknown noise points, improving robustness.
Hence, the above results further verify the effectiveness of the proposed granular-ball kernels.

\subsection{Runtime Analysis}   
Table~\ref{time} presents the runtime comparison between the GB-induced framework and six baselines. The results on five representative datasets show that the average speed of GB-MKKM is 0.170 seconds higher than that of MKKM, and the average speed of GB-MKKM-SR is 0.201 seconds higher. Similarly, GB-SMKKM is 1.491 seconds faster than SMKKM respectively.  
Note that granular-ball construction can be done in advance independently of the optimization of different models. Thus, we do not include the construction time here.  
In section 4 of the Appendix, we report both ball generation and model optimization time. The results show that the speed increases as the dataset scale increases. 
Therefore, granular-ball-induced kernel matrix construction scales the MKKM algorithms.  
Meanwhile, in the section 5 of the appendix, convergence graphs of the monotonically decreasing target values of multiple datasets during the iterative process are presented.


\subsection{Analysis of Granular-ball Numbers}

Computing a full $n \times n$ kernel matrix in traditional MKKM methods is computationally expensive. 
The GB-MKKM framework uses granular balls to characterize data distribution.
By reducing the size of the kernel matrix, GB-MKKM minimizes the complexity of subsequent optimization, avoiding modeling the point-to-point dependency inherent in traditional methods. In Figure \ref{number}, we report the number of granular balls adaptively optimized by the granular-ball computing process. 

As can be seen, the granular-ball number is significantly smaller than the original number of sample points on the same datasets. The results indicate that the GB-MKKM framework not only achieves promising high clustering performance, but also effectively reduces the number of sample points. This demonstrates its remarkable flexibility and applicability to various MKKM algorithms, fully proving the ability of granular balls to characterize data in kernel space effectively. 
Other experimental results about the single kernel method and other methods are supplemented in Section 6 of the Appendix.

\begin{figure}[t]
    \centering
    \begin{subfigure}[b]{0.42\textwidth}
    \centering    \includegraphics[width=\textwidth, trim=0 0 0 0, clip]{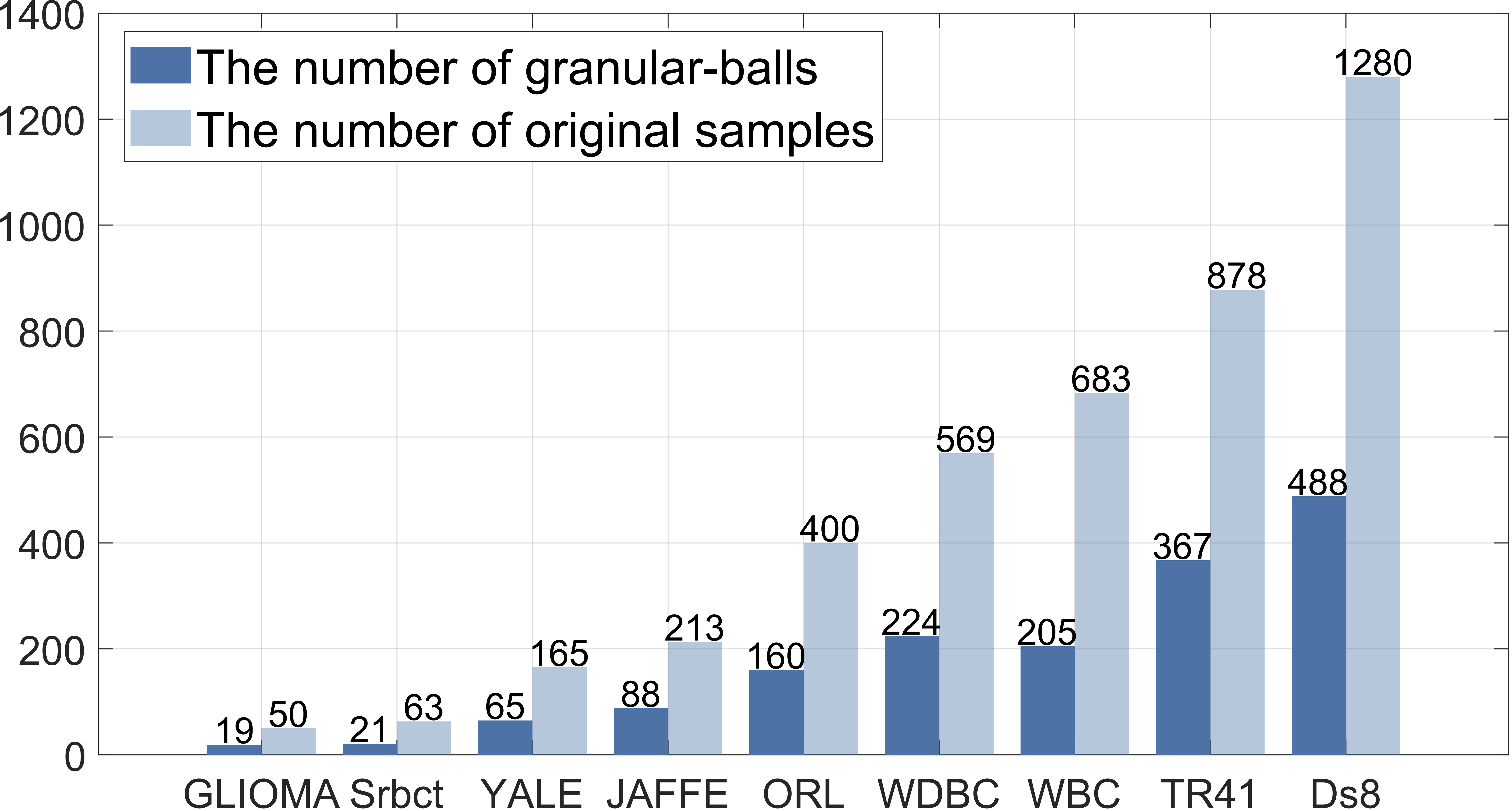}
    \end{subfigure} 
    \caption{Granular-balls count vs. sample size.}
    \label{number}
\end{figure}

\section{Conclusion}

This paper introduces the GB-induced clustering framework, which enhances multi-kernel k-means (MKKM) by incorporating granular-ball computing. Traditional MKKM suffers from high computational costs due to full kernel matrix construction and point-to-point dependencies. To address this, we propose a novel granular-ball representation that models relationships between local data regions (granular balls) instead of individual points. This ball-to-ball approach significantly reduces complexity, improves robustness to noise, and achieves higher clustering accuracy. 
Extensive experiments show that GB-induced and its variants consistently outperform traditional MKKM methods across diverse datasets, demonstrating strong generalization, robustness, and scalability. The framework is broadly applicable to existing MKKM algorithms and can be extended to deep architectures for enhanced multi-kernel clustering in neural networks.


\section*{Acknowledgments}
This research was supported by the Chongqing Graduate Research Innovation Program CYB25259
and the National Natural Science Foundation of China under Grant Nos. 62221005, 62450043, 62222601, and 62176033.

\bibliographystyle{named}
\bibliography{ijcai25}

\end{document}


\title{Supplementary Materials of Granular-Ball-Induced Multiple Kernel K-Means}
\maketitle

\section{Granular-ball Generation Process}
This part presents the process and details of granular-ball update.  First, Figure~\ref{granular1} provided a schematic illustration of granular-ball splitting. Figure~1 illustrates that, for a granular-ball \( GB_i \), if its $\text{CCM}_{gb}$ value exceeds twice the $\text{CCM}_{\text{median}}$ value of all granular-balls in the current splitting stage, the splitting process is terminated. Otherwise, it is further divided into two sub-balls and proceeds to the next splitting stage. The reason for using the median as the threshold criterion is that it provides a robust estimate of the central tendency, which is less sensitive to outliers and skewed distributions, ensuring a more stable and reliable splitting decision. Secondly, Algorithm 1 provides the pseudo-code for granular-ball update.

\begin{algorithm}[H]\footnotesize
    \caption{Granular-Ball Generation Process}
    \label{alg:GBKernel}
    \textbf{Input}: Dataset $\mathcal{D} = \{\mathbf{x}_1, \mathbf{x}_2, \dots, \mathbf{x}_n\}$. \\
    \textbf{Output}:The final $GB$.
    \begin{algorithmic}[1] 
        \STATE Treat the entire dataset as a whole.
        \STATE Coarsely divide the data into initial granular-balls $GB_i (i = 1, 2, \dots, \sqrt{n} )$ by k-means.
        \WHILE{the granular-ball set is not stable}
          \FOR{each $GB_i \in \{GB_1, GB_2, \dots, GB_{\sqrt{n}}\}$}
                \STATE Compute $\text{CCM}_{GB_i}$ using Equation 3 and its median value $\text{CCM}_{\text{median}}$.
                \IF{$\text{CCM}_{GB_i} < 2 \cdot  \text{CCM}_{\text{median}}$}
                    \STATE Call the 2-means algorithm to split $GB_i$ into two sub-balls $GB_{i1}$ and $GB_{i2}$.
                    \IF{$|GB_{i1}| > 1$ and $|GB_{i2}| > 1$}
                        \STATE Add $GB_{i1}$ and $GB_{i2}$ to the granular-ball set.
                    \ELSE
                        \STATE Retain the original $GB_i$.
                    \ENDIF
                \ENDIF
            \ENDFOR
        \ENDWHILE
        \STATE Get the final set of granular-ball.
        \STATE \textbf{return} ${GB}$.
    \end{algorithmic}
\end{algorithm}

\begin{figure}[h!]
    \centering
    \centering  \includegraphics[width=0.5\textwidth, trim=0 0 0 0, clip]{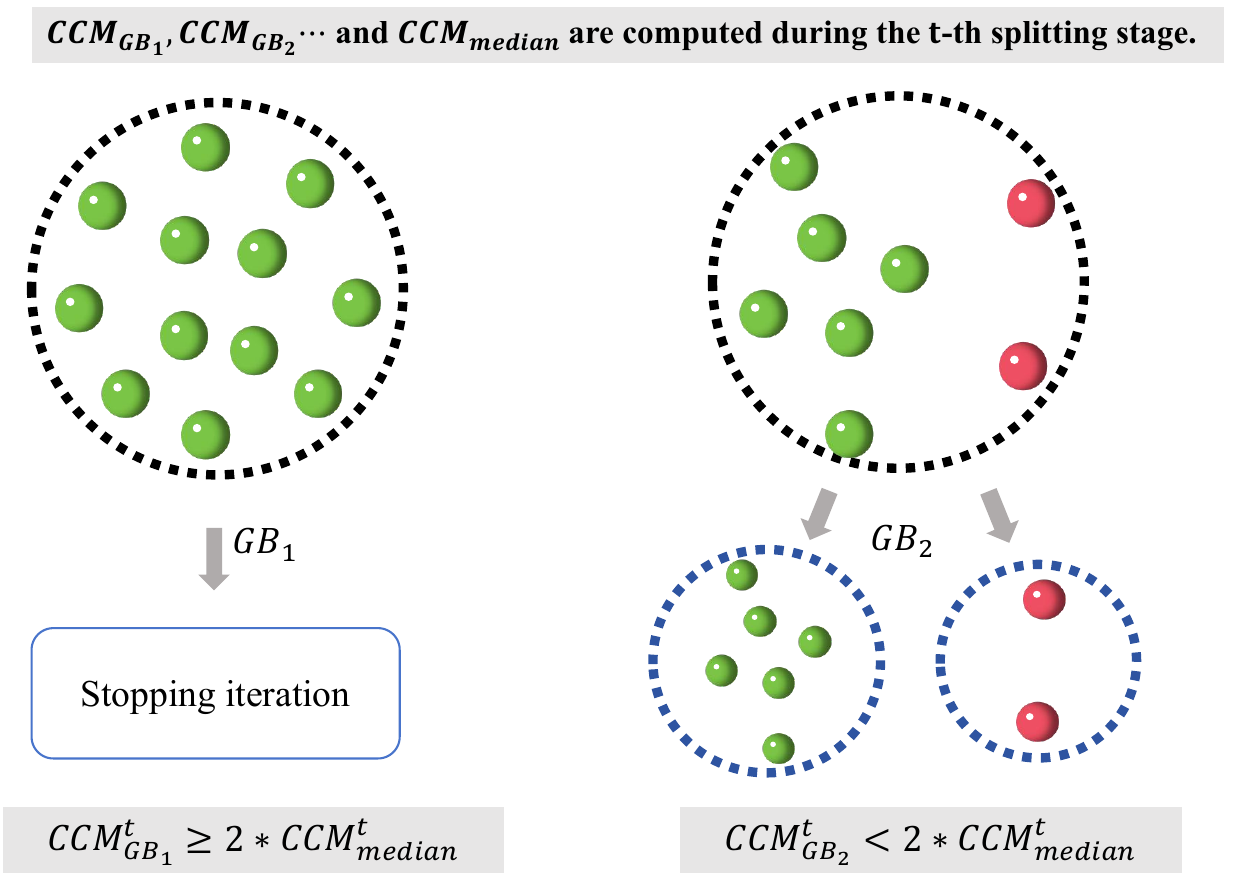}
    \caption{Examples of $\text{CCM}_{gb}$.}
    \label{granular1}
    \vspace{-5pt} 
\end{figure}
\section{Proof of Differentiability and Convexity}

The objective function is expressed as follows:  
\begin{equation}\label{obj}
\Gamma(\sigma)=\max_{\mathbf{H}_{gb} \in \Omega} \operatorname{Tr} \left[ \mathbf{K}^{GB}_{\sigma} ( \mathbf{H}_{gb} \mathbf{H}_{gb}^T - \mathbf{I}_m ) \right].
\end{equation}

Equation (\ref{obj}) can be optimized through numerical optimization methods, leading to a stable and globally optimal solution. Below, specific proofs are provided in Theorem \ref{th1} and Theorem \ref{th2}.
\begin{theorem}\label{th1}
 The objective function $\Gamma(\sigma)$ is differentiable.
\end{theorem}
\begin{proof}
First, since $\mathbf{K}_{\boldsymbol{\sigma}}^{GB} = \sum_{k=1}^\omega \sigma_k^2 {\mathbf{K}}_k^{GB}$ is a linear combination of $\sigma_k$, we have for every $\sigma_k$:
\begin{equation}
\frac{\partial \mathbf{K}_\sigma^{GB}}{\partial \sigma_k} = 2\sigma_k \mathbf{K}_k^{GB},
\end{equation}
which shows that $\mathbf{K}_\sigma^{GB}$ is differentiable with respect to $\sigma_k$.

Second, since $\Gamma(\sigma)$ depends on $\mathbf{K}_\sigma^{GB}$, using the chain rule, we have:
\begin{align}
\frac{\partial \Gamma(\sigma)}{\partial \sigma_k} = \frac{\partial}{\partial \sigma_k} \max_{\mathbf{H}_{gb} \in \Omega} \mathrm{Tr}[\mathbf{K}_\sigma^{GB}(\mathbf{H}_{gb} \mathbf{H}_{gb}^T - \mathbf{I}_m)]\nonumber\\
= \max_{\mathbf{H}_{gb} \in \Omega} \frac{\partial}{\partial \sigma_k} \mathrm{Tr}[\mathbf{K}^{GB}_\sigma(\mathbf{H}_{gb} \mathbf{H}_{gb}^T - \mathbf{I}_m)],\nonumber
\end{align}
and substituting $\mathbf{K}_{\boldsymbol{\sigma}}^{GB} = \sum_{k=1}^\omega \sigma_k^2 {\mathbf{K}}_k^{GB}$, we get:
\begin{equation}
\frac{\partial \Gamma(\sigma)}{\partial \sigma_k} = 2\sigma_k \max_{\mathbf{H}_{gb} \in \Omega} \mathrm{Tr}[\mathbf{K}_k^{GB}(\mathbf{H}_{gb} \mathbf{H}_{gb}^T - \mathbf{I}_m)].
\end{equation}
Therefore, $\Gamma(\sigma)$ is also differentiable with respect to $\sigma_k$, and the form of its partial derivative depends only on the granular-ball kernel matrix $\mathbf{K}_\sigma^{GB}$ and the granular-ball distribution matrix $\mathbf{H}_{gb}$. Theorem 2 indicates that the objective function has sufficient smoothness to guide the update direction and step size of parameter optimization.
\end{proof}

In summary, the objective function $\Gamma(\sigma)$ based on the granular-ball kernel matrix exhibits excellent smoothness and convexity. The multi-granularity representation of granular-balls and the optimization of the parameter $\sigma_k$ provide a more flexible capability for multi-kernel data modeling. 

\begin{theorem}\label{th2}
The objective function $\Gamma(\sigma)$ is convex.
\end{theorem}

\begin{proof}
For $\sigma_1, \sigma_2 \in \Delta$ and $0 < \alpha < 1$, we have:
\begin{align}
    &\scalebox{0.9}{$\Gamma[\alpha \sigma_1 + (1-\alpha)\sigma_2]$}\\\nonumber
    &=\scalebox{0.9}{$ \alpha \Gamma(\sigma_1) + (1-\alpha) \Gamma(\sigma_2)$} \\ \nonumber
    &= \scalebox{0.8}{$\alpha \sigma_1 + (1-\alpha)\sigma_2 \max_{\mathbf{H}_{gb} \in \Omega} \mathrm{Tr}[\mathbf{K}_\sigma^{GB} (\mathbf{H}_{gb}\mathbf{H}_{gb}^T - \mathbf{I}_m)]$} \\\nonumber
    &=\scalebox{0.8}{$\max_{\mathbf{H}_{gb} \in \Omega} \mathrm{Tr} \left[\sum_{k=1}^\omega \left(\alpha \sigma_{1,k}^2 + (1-\alpha)^2 \sigma_{2,k}^2\right) \mathbf{K}_k^{GB}(\mathbf{H}_{gb} \mathbf{H}_{gb}^T - \mathbf{I}_m) \right].$}
\end{align}

Since  $\sigma_{1,k}, \sigma_{2,k} \geq 0$, based on the properties of linear combinations and matrix inequalities in convex optimization, we can obtain:

\begin{align}
&\scalebox{0.65}{$\max_{\mathbf{H}_{gb} \in \Omega} \operatorname{Tr} \Bigg\{ 
\sum_{k=1}^\omega \big[ \alpha^2 \sigma_{1,k}^2 
+ 2\alpha(1{-}\alpha)\sigma_{1,k}\sigma_{2,k} 
+ (1{-}\alpha)^2 \sigma_{2,k}^2 \big] \cdot\ \mathbf{K}_{k}^{GB} 
(\mathbf{H}_{gb} \mathbf{H}_{gb}^T - \mathbf{I}_m) 
\Bigg\}$} \nonumber \\
&\scalebox{0.8}{$\leq \max_{\mathbf{H}_{gb} \in \Omega} \operatorname{Tr} \Bigg\{ 
\sum_{k=1}^\omega \big[ \alpha^2 \sigma_{1,k}^2 
+ (1{-}\alpha)^2 \sigma_{2,k}^2 \big] \cdot\ \mathbf{K}_{k}^{GB}
(\mathbf{H}_{gb} \mathbf{H}_{gb}^T - \mathbf{I}_m) 
\Bigg\} $} \nonumber \\
&\scalebox{0.8}{$\leq \alpha \max_{\mathbf{H}_{gb} \in \Omega} \operatorname{Tr} \Bigg\{ 
\sum_{k=1}^\omega \sigma_{1,k}^2 \mathbf{K}_{k}^{GB}
(\mathbf{H}_{gb} \mathbf{H}_{gb}^T - \mathbf{I}_m) 
\Bigg\}$} \nonumber \\
&\scalebox{0.8}{$+ (1{-}\alpha) \max_{\mathbf{H}_{gb} \in \Omega} \operatorname{Tr} \Bigg\{ 
\sum_{k=1}^\omega \sigma_{2,k}^2 \mathbf{K}_{k}^{GB}
(\mathbf{H}_{gb} \mathbf{H}_{gb}^T - \mathbf{I}_m) 
\Bigg\}$}. \label{eq:inequality}
\end{align}

Simplifying, we derive:
\begin{equation}
\Gamma[\alpha \sigma_1 + (1-\alpha)\sigma_2] \leq \alpha \Gamma(\sigma_1) + (1-\alpha) \Gamma(\sigma_2).
\end{equation}
Thus, $\Gamma(\sigma)$ is convex. Theorem 3 demonstrates that the objective function has a globally optimal solution.
\end{proof}

\section{Supplementary of Full Results.}

\section{Friedman Test}

We employ the well-known Friedman test to compare and analyze the significant differences among 12 algorithms across 9 datasets. The average ranking and accuracy of these algorithms on the 9 datasets are presented in Table~\ref{tab:Friedman test}.

\begin{table}[h!] 
\centering
\caption{Average accuracy and ranks of 12 algorithms on 9 datasets}
\label{tab:Friedman test}
\begin{tabular}{l|c|c}
\toprule
\textbf{Method} & \textbf{Avg.ACC} & \textbf{Avg.rank}  \\\midrule
Random Selection       & 58.84         & 9.44              \\
Kmeans       & 64.73         & 7.56                   \\
BKHK    & 67.21         & 7.11                \\
RKKM-a        & 62.70         & 9.33                    \\
AASC            & 49.48         & 10.33                   \\
RMKKM           & 71.98         & 5.56                   \\
MKKM            & 72.79         & 4.78             \\
GB-MKKM & 73.84         & 5.22                  \\
MKKM-SR        & 66.54         & 7.44             \\
GB-MKKM-SR & 71.93     & 5.56        \\
SMKKM           & 73.71         & 3.89                 \\
GB-SMKKM & 80.54        & 1.78 \\
\bottomrule
\end{tabular}
\end{table}

Next, the formula for Friedman statistical variables is as follows:
\begin{equation}\label{105}
\chi_{F}^{2}=\frac{12n}{H(H+1)}[\sum_{i}^{H}R_{i}^{2}-3n(H+1)]=48.80,
\end{equation}
where $H$ is the number of algorithms and $n$ is the number of the datasets.  $R_{i}$ represents the average ranking of the $i$ algorithm on the 9 data sets. In addition, according to the $\chi_{F}^{2}$ distribution with $H-1$ degrees of freedom.

\begin{equation}\label{106}
F_{F}=\frac{(n-1)\chi_{F}^{2}}{n(H-1)-\chi_{F}^{2}}=7.78,
\end{equation}
where $F_{F}((H-1),(H-1)(n-1))$ obeys the F-distribution, and its degree of freedom is $(H-1)$ and $(H-1)(n-1)$. In this paper, we choose $\alpha=0.05$ and we can get $F_{\alpha}(11,88)=2.07$. Obviously, $F_{F}$ is larger than $F_{\alpha}$, So we can reject the null hypothesis at a confidence level of 95\%, indicating that the 12 algorithms proposed in this paper are significantly different and not random.

Next, we use the Nemenyi post-hoc test to compare the errors of the 12 algorithms. If the average rank difference between two algorithms exceeds the critical value, their performance is considered significantly different. A larger difference indicates a more pronounced performance gap. The calculated threshold is \( q_{\alpha} = 3.268 \), and the critical difference (CD) is computed as follows:

\begin{equation}\label{106}
\scalebox{0.9}{$CD=q_{\alpha=0.05}\sqrt{\frac{H(H+1)}{6n}}=3.268\times\sqrt{\frac{12(12+1)}{12\times9}}=5.56$}.
\end{equation}

\section{Convergence Analysis}  
Convergence serves as a fundamental criterion for validating the theoretical soundness of the proposed algorithm. The convergence behaviors on the four datasets are illustrated in Figure 2.

\begin{figure}[h!]
    \centering
    \begin{subfigure}[b]{0.23\textwidth} 
    \centering    \includegraphics[width=\textwidth]{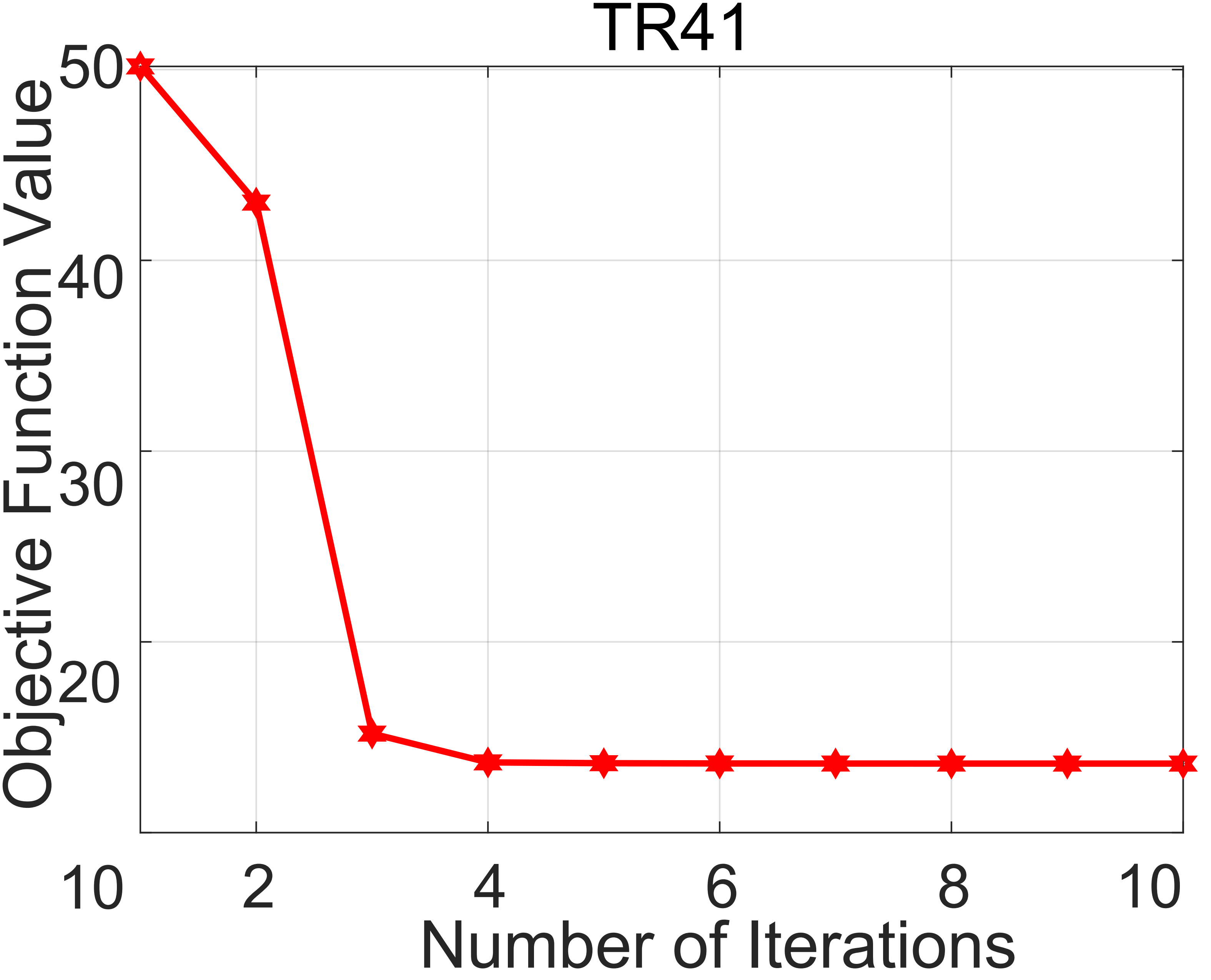}
    \end{subfigure}
    \begin{subfigure}[b]{0.23\textwidth}
    \centering    \includegraphics[width=\textwidth]{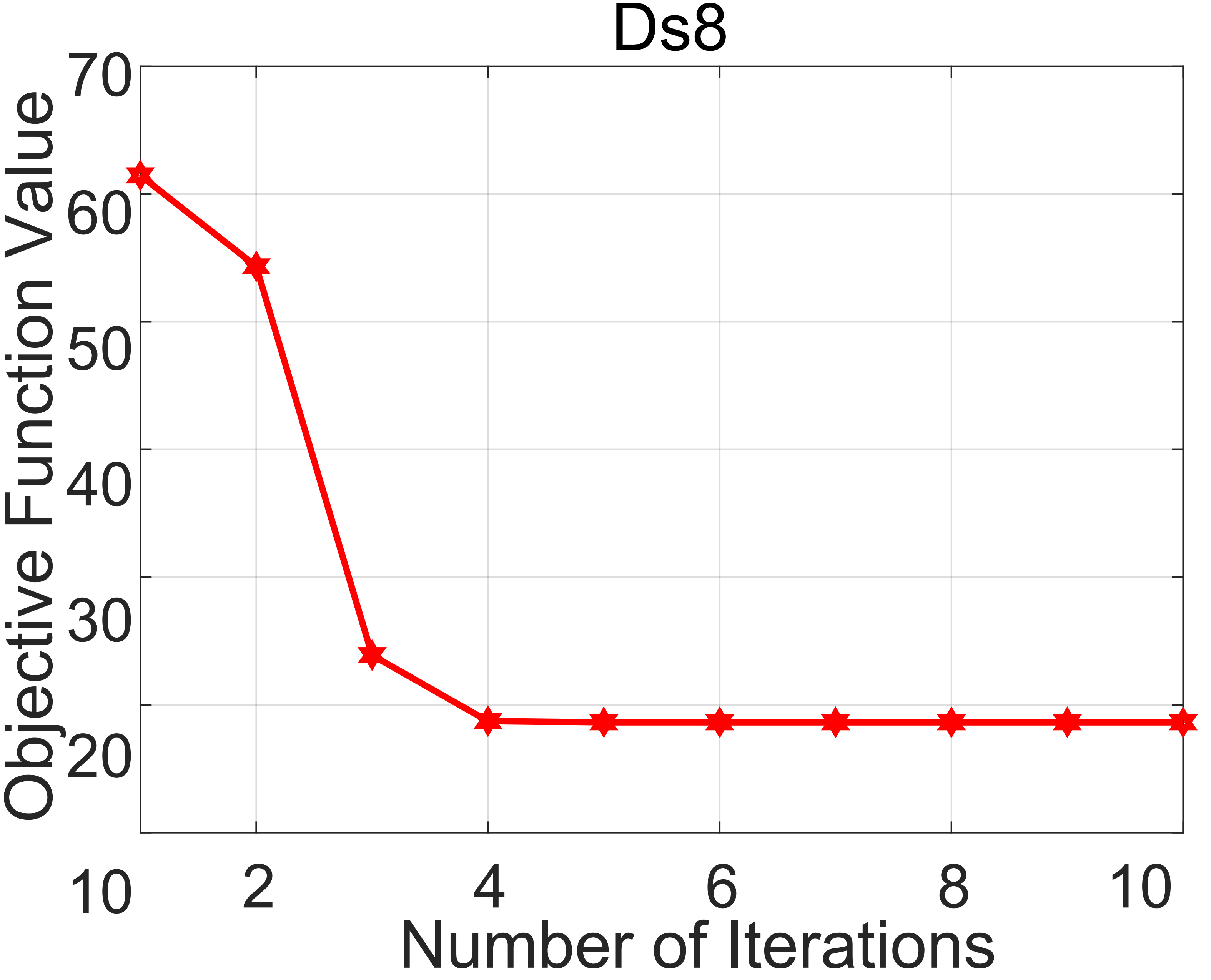}
    \end{subfigure}
    \begin{subfigure}[b]{0.23\textwidth}
    \centering    \includegraphics[width=\textwidth]{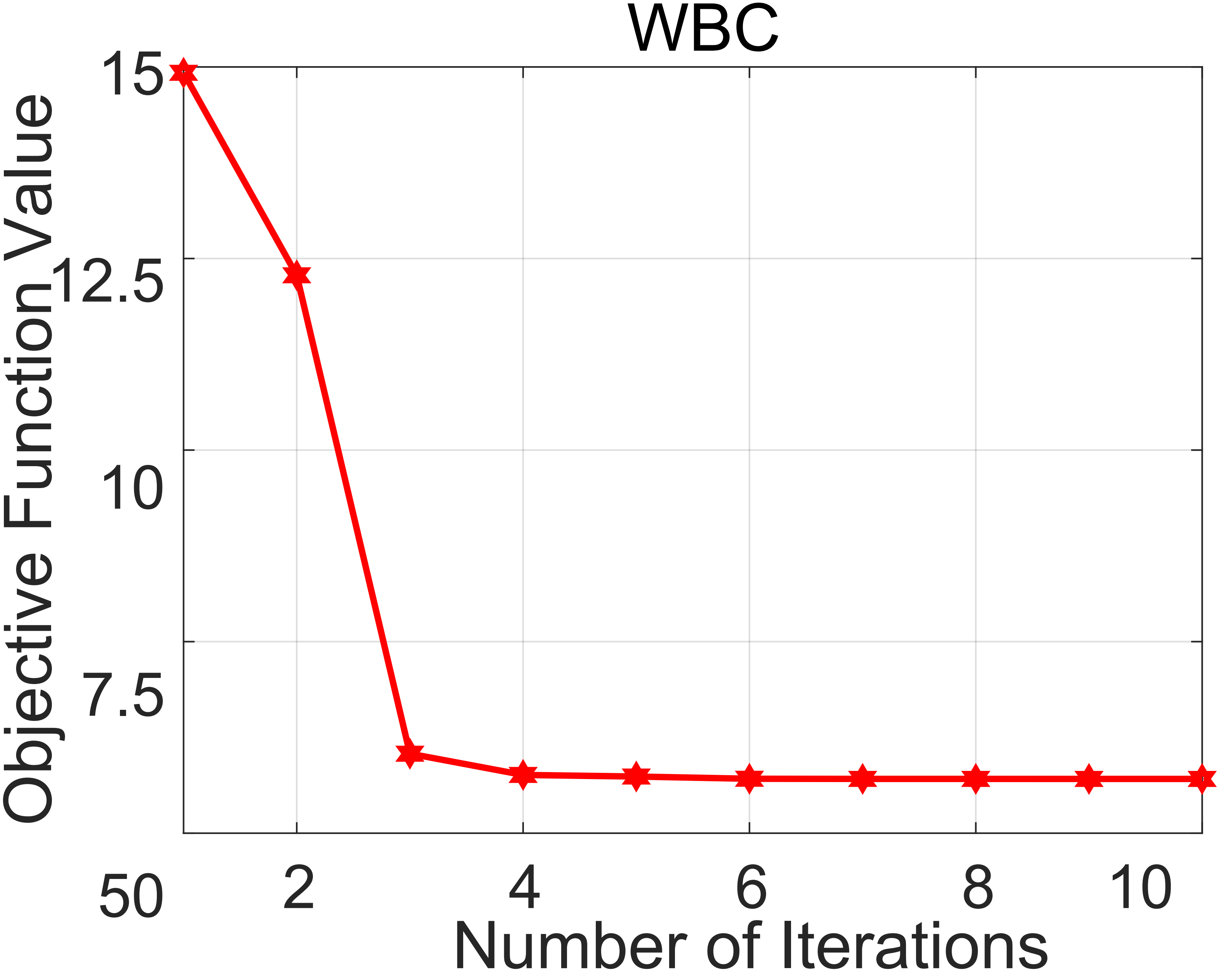}
    \end{subfigure}
    \begin{subfigure}[b]{0.23\textwidth}
    \centering    \includegraphics[width=\textwidth]{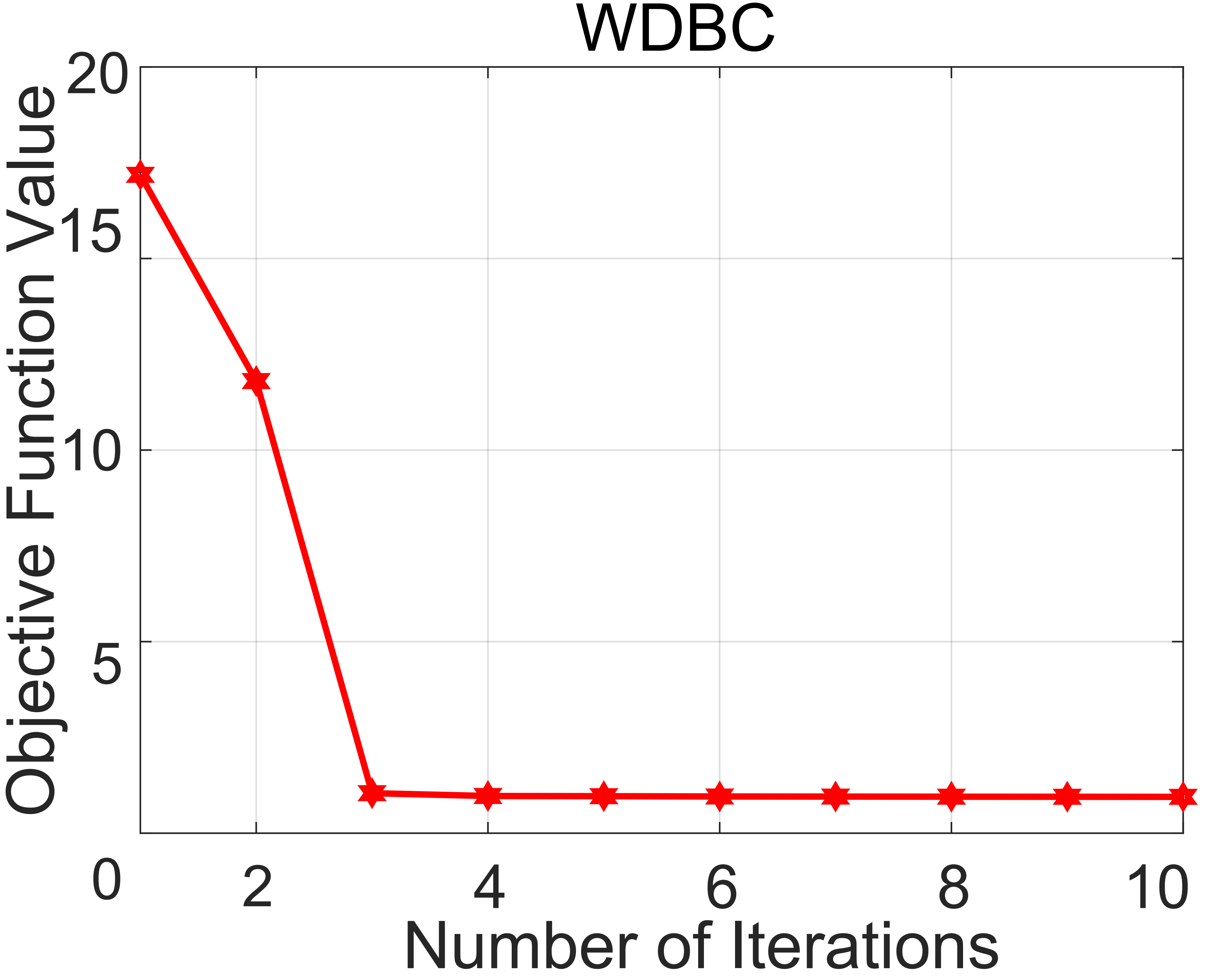}
    \end{subfigure}
    \caption{\small The iteration-wise changes in the objective values of GB-SMKKM on datasets.}
    \label{objective vaule}
\end{figure}

As illustrated in Figure 2, the objective function exhibits a monotonically decreasing trend during the iterative process, which empirically confirms the convergence of the proposed algorithm.

\section{Other Experimental Results }  
\subsection{The Running Time Include Granular-ball Constuction}  
\begin{table}[h!]
\centering
\caption{{Running time (s).}}
\label{tab:granular ball time.}
\renewcommand{\arraystretch}{0.85} 
\resizebox{\linewidth}{!}{
\begin{tabular}{cccccc}
\toprule
\textbf{Dataset} & \textbf{$N$} & \textbf{dims} &\textbf{SMKKM} & \textbf{Ball-generation} & \textbf{GB-SMKKM} \\
\midrule
COIL20 & 1440 & 1024 & 8.06 & 4.12 & 3.84 \\
segment & 2310& 18 & 17.46 & 4.27 & 12.38 \\
misplice-2 & 3175& 2  & 34.12 & 8.98 & 6.52 \\
Phoneme & 5404& 5 & 130.29 & 17.20 & 28.30 (much faster)\\
\bottomrule
\end{tabular}
}
\end{table}
Granular-ball construction is performed as a pre-processing step that is independent of the subsequent model optimization process. Therefore, we do not include its runtime in the main performance comparison.  As shown in Table\~1, both the ball generation and the overall optimization processes remain efficient, and the total runtime improves as the dataset size increases. This demonstrates the scalability and practicality of our method for large-scale applications.

\subsection{Experiments on KKM}
Although our method is primarily designed to enhance and extend the Multiple Kernel k-means (MKKM) framework, we also include Kernel k-means (KKM) in our experiments for completeness and deeper insight. As a classical single-kernel clustering method, KKM provides a strong baseline for evaluating the benefits of multiple kernel integration. Including KKM serves two purposes: (1) it allows us to demonstrate that our improvements are not only effective in the multi-kernel setting, but also outperform strong single-kernel methods; and (2) it helps isolate the specific advantages brought by the granular-ball construction and adaptive integration mechanisms, beyond the general benefits of kernelization. Therefore, supplementing the experimental results with KKM adds value by highlighting the necessity of moving from traditional kernel clustering toward more flexible and robust multi-kernel solutions like ours.
Table \ref{tab:kkm} reports the clustering performance of KKM.
\begin{table}[h!]
\centering
\caption{{GB-KKM and KKM with the linear kernel (regarding ACC).}}
\label{tab:kkm}
\resizebox{\linewidth}{!}{
\begin{tabular}{ccccc}
\toprule
\textbf{Dataset} & \textbf{$N$} & \textbf{dims} &\textbf{KKM} & \textbf{GB-KKM}  \\
\midrule
movement-libras & 360 & 90  & 46.58$\pm$0.027 & 54.23$\pm$0.034  \\
Yeast & 1484 & 7   & 32.91$\pm$0.003 & 43.47$\pm$0.004 \\
misplice-2 & 3175& 2    & 71.62$\pm$0.001& 71.91$\pm$0.008\\
ls & 6435& 36  & 65.01$\pm$0.019 & 66.47$\pm$0.019 \\
\bottomrule
\end{tabular}
}
\end{table}

\begin{table*}[h!] 
\centering
\caption{Comparison results in terms of ACC on nine datasets.}
\label{tab7}
\resizebox{0.8\textwidth}{!}{ 
\begin{tabular}{l|c|c|c|c|c|c|c|c|c|c}
\toprule
\textbf{Method} & \textbf{YALE} & \textbf{JAFFE} & \textbf{ORL}  & \textbf{TR41} & \textbf{Srbct} & \textbf{GLIOMA} & \textbf{WBC} & \textbf{WDBC} & \textbf{Ds8} & \textbf{Average} \\\midrule
Random Selection       & 50.30         & 58.69          & 65.50                 & 48.86         & 47.62          & 46.00          & 81.72        & 62.57         & 68.36        &58.84         \\

Kmeans       & 43.03         & 80.28          & 65.00                 & 55.47         & 52.38          & 64.00          & 81.55        & 87.70         & 53.13        &64.73          \\
BKHK    & 51.52         & 53.05          & 67.25                 & 44.53         & 68.25          & 64.00         & 95.02         &83.66         & 77.58         & 67.21          \\

RKKM-a        & 41.06         & 62.77          & 48.15               & 47.84         & 49.84          & 50.90          & 80.40         & 85.64         & 92.29         & 62.10           \\
AASC            & 40.64         & 30.99          & 27.09              & 19.68         & 34.76          & 53.90          & 76.57         & 62.92         & 98.75         & 49.48           \\
RMKKM           & 52.18         & 87.07          & 55.58               & 62.85         & 53.89          & 59.30          & 96.49         & 86.95         & 93.53         & 71.98           \\
MKKM            & 51.55         & 93.90          & \textbf{68.36}         & 56.51         & 68.65          & 59.70          & 95.51         & 88.05         & 72.88         & 72.79           \\
GB-MKKM & 53.79         & 95.11          & 62.53          & 56.54         & 78.64          & 65.00          & 95.21         & 83.78         & 73.94         & 73.84           \\
MKKM-SR        & 47.27         & 79.34          & 59.75                 & 60.48         & 68.25          & 56.00          & 60.32         & 86.99         & 80.47         & 66.54           \\
GB-MKKM-SR & 47.69     & 96.55          & 65.85                & 52.50         & 68.52          & 52.63          & 95.78         & 87.10         & 80.76         & 71.93           \\
SMKKM           & \textbf{54.30}         & 95.87          & 66.69                 & 56.65         & 66.27          & 59.40          & 96.93         & 88.40         & 78.87         & 73.71           \\
GB-SMKKM & 53.46        & \textbf{97.00} & 65.34                 & \textbf{64.63} & \textbf{82.14} & \textbf{74.47} & \textbf{97.56} & \textbf{91.07} & \textbf{99.18} & \textbf{80.54}  \\
\bottomrule
\end{tabular}
}
\end{table*}

\begin{table*}[h!] 
\centering
\caption{Comparison results in terms of NMI on nine datasets.}
\label{tab8}
\resizebox{0.8\textwidth}{!}{ 
\begin{tabular}{l|c|c|c|c|c|c|c|c|c|c}
\toprule
\textbf{Method} & \textbf{YALE} & \textbf{JAFFE} & \textbf{ORL}  & \textbf{TR41} & \textbf{Srbct} & \textbf{GLIOMA} & \textbf{WBC} & \textbf{WDBC} & \textbf{Ds8} & \textbf{Average} \\\midrule
Random Selection       & 52.00         & 73.26          & 79.27                 & 41.14         & 36.50          & 29.24          & 39.68        & 19.51         & 35.49        &45.12          \\
Kmeans       & 51.63         & 69.05          & 81.49                 & 38.32         & 49.05          & 40.84         & 69.70         &37.92         & 35.84         & 52.65          \\
BKHK       & 43.69         & 82.38          & 77.57                 & 53.39         & 41.74          & 49.37         & 72.91        & 47.55         & 57.01        & 58.40           \\

RKKM-a        & 46.01         & 70.17          & 68.44                 & 42.91         & 29.73          & 28.22          & 41.09         & 44.93         & 82.27         & 50.42           \\
AASC            & 46.83         & 27.08          & 43.65                 & 5.88          & 5.30           & 43.36          & 15.91         & 0.26          & 92.41         & 31.19          \\
RMKKM           & 55.58         & 89.37          & 74.84                  & 63.52         & 31.35          & 49.82          & 76.51         & 46.05         & 76.16         & 62.58           \\
MKKM            & 52.98         & 93.96          & 82.44                 & 57.94         & 50.65          & 40.65          & 82.13         & 47.09         & 58.37         & 62.91           \\
GB-MKKM & \textbf{65.76}         & 94.82          & 82.20                & 60.99         & 69.10          & 53.55          & 73.29         & 35.38         & 62.37         & 66.38           \\
MKKM-SR        & 51.41         & 85.55          & 77.48                  & 61.60         & 55.43          & 36.70          & 2.13          & 41.63         & 52.26         & 51.58           \\
GB-MKKM-SR & 62.08     & 95.76          & 82.92             & 59.16         & 56.45          & 44.08          & 74.76         & 41.16         & 53.94         & 63.37           \\
SMKKM           & 56.11         & 95.18          & 82.06                & 57.57         & 47.85          & 41.20          & 79.71         & 45.86         & 57.71         & 62.58           \\
GB-SMKKM & 65.09        & \textbf{97.80} & \textbf{83.89}         & \textbf{64.07} & \textbf{82.31} & \textbf{61.22}          & \textbf{85.75}         & \textbf{53.67} & \textbf{94.76} & \textbf{76.51}  \\

\bottomrule
\end{tabular}
}
\end{table*}
\vspace{-500mm}

\begin{table*}[h!] 
\centering
\caption{Comparison results in terms of ARI on nine datasets.}
\label{tab9}
\resizebox{0.8\textwidth}{!}{ 
\begin{tabular}{l|c|c|c|c|c|c|c|c|c|c}
\toprule
\textbf{Method} & \textbf{YALE} & \textbf{JAFFE} & \textbf{ORL}  & \textbf{TR41} & \textbf{Srbct} & \textbf{GLIOMA} & \textbf{WBC} & \textbf{WDBC} & \textbf{Ds8} & \textbf{Average} \\\midrule
Random Selection       & 27.58        & 49.94          & 49.18                 & 23.34         & 23.46          & 14.88          & 40.14        & 4.77         & 26.53        &28.87         \\
Kmeans       & 27.55         & 47.07          & 53.31                 & 14.96         & 31.26          & 27.97         & 80.83         &44.05         & 40.06         & 40.78           \\
BKHK        & 20.38         & 72.92          & 48.94                 & 36.85         & 22.82          & 36.04          & 83.55        & 56.11         & 55.33        & 48.10           \\

RKKM-a       & 43.58         & 66.83          & 52.85                  & 63.95         & 56.19          & 53.40          & 80.40         & 85.64         & 96.64         & 66.61           \\
AASC            & 42.33         & 32.51          & 31.49                  & 30.40         & 43.97          & 57.00          & 76.57         & 62.92         & 98.75         & 52.88           \\
RMKKM           & 53.64         & 88.90          & 60.20                  & 77.61         & 56.67          & 64.30          & 96.47         & 86.95         & 93.53         & 75.36           \\
MKKM            & 52.79         & 93.99          & 71.76                 & 74.83         & 70.79          & 62.00          & 97.51         & 88.05         & 92.97         & 78.30           \\
GB-MKKM & 56.44         & 95.29          & 69.26                 & 77.06         & 78.86          & 67.50          & 95.21         & 83.78         & 93.25         & 79.63           \\
MKKM-SR        & 50.30         & 84.04          & 64.00                  & 76.54         & 69.84          & 58.00          & 65.01         & 86.99         & 87.42         & 71.35           \\
GB-MKKM-SR & 49.23     & 95.76          & 70.73                  & 71.94         & 78.95          & 57.89          & 95.78         & 87.10         & 87.37         & 77.19           \\
SMKKM           & 55.00         & 95.87          & 70.55                  & 74.46         & 68.25          & 63.30          & 96.93         & 88.40         & 88.93         & 77.97           \\
GB-SMKKM & \textbf{56.92}        & \textbf{97.95} & \textbf{72.00}             & \textbf{78.61}         & \textbf{87.38} & \textbf{74.47}          & \textbf{97.56} & \textbf{91.07} & \textbf{99.18} & \textbf{83.90}  \\
\bottomrule
\end{tabular}
}
\end{table*}